\documentclass{article}

\usepackage[margin=1.17in]{geometry}

\usepackage{ulem}

\usepackage[round]{natbib}
\PassOptionsToPackage{numbers, sort&compress}{natbib}

\usepackage[dvipsnames]{xcolor}

\usepackage{graphicx}
\usepackage{subcaption}
\usepackage{xfrac}
\usepackage{old-arrows}

\usepackage{wrapfig}

\usepackage{empheq}

\usepackage{commath}
\usepackage{mathtools}
\usepackage{bm}
\mathtoolsset{showonlyrefs}
\usepackage{algorithm}
\usepackage{algorithmic}
\usepackage[dvipsnames]{xcolor}

\usepackage{enumitem}
\usepackage[utf8]{inputenc} % allow utf-8 input
\usepackage[T1]{fontenc}    % use 8-bit T1 fonts
\usepackage{hyperref}       % hyperlinks
\usepackage{url}            % simple URL typesetting
\usepackage{booktabs}       % professional-quality tables
\usepackage{amsfonts}       % blackboard math symbols
\usepackage{nicefrac}       % compact symbols for 1/2, etc.

\usepackage{microtype}      % microtypography
\usepackage{xcolor}         % colors
\usepackage{xfrac}
\usepackage{wrapfig}
\usepackage{enumitem}
\usepackage{amsmath}
\usepackage{amssymb}
\usepackage{enumitem}

\newlist{assumptions}{enumerate}{1}
\setlist[assumptions]{label=(A\arabic*), ref=A\arabic*} % or ref=(A\arabic*) if you want the parens in \ref

\usepackage{hyperref}
 \hypersetup{
     colorlinks=true,
     linkcolor=blue,
     filecolor=blue,
     citecolor = blue,      
     urlcolor=magenta,
 }

\usepackage{amssymb}

\usepackage{amsthm}
\usepackage{amsmath}
\usepackage{mathtools}
\mathtoolsset{showonlyrefs}
\usepackage{dsfont}
\usepackage{amssymb}
\usepackage{latexsym}
\usepackage{verbatim}
\usepackage{xfrac}
\usepackage{physics}
\usepackage{enumitem}
\usepackage{stmaryrd}

\usepackage{bbm}

\usepackage{amsthm}

{\begin{list}               %    with flush left bullets.
    {$\bullet$ \hfill}{
        \setlength{\leftmargin}{\parindent}
        \setlength{\parsep}{0.04\baselineskip}
        \setlength{\itemsep}{0.5\parsep}
        \setlength{\labelwidth}{\leftmargin}
        \setlength{\labelsep}{0em}}
    }
{\end{list}}

\providecommand{\cref}[1]{Chapter~\ref{chap:#1}}

\providecommand{\R}{\ensuremath{\mathbb{R}}}
\providecommand{\C}{\ensuremath{\mathbb{C}}}

\providecommand{\norm}[1]{\lVert#1\rVert}

\providecommand{\inprod}[1]{\left\langle#1\right\rangle}

\providecommand{\vect}[1]{\mathrm{vec}\left[#1\right]}

\newcommand{\E}{\mathbb{E}}
\renewcommand{\P}[1]{\mathbb{P}\left[#1\right]}
\newcommand{\Var}[1]{\mathrm{Var}\left[#1\right]}
\newcommand{\Ea}[1]{\E\left[#1\right]}
\newcommand{\Eb}[2]{\E_{#1}\left[#2\right]}
\newcommand{\module}[1]{\left|#1\right|}
\newcommand{\length}[1]{\mathcal{P}_{#1}}

\newcommand{\polylog}{\mathrm{polylog}}

\providecommand{\spec}[1]{\mathrm{spec}\left[#1\right]}
\providecommand{\smi}{{\setminus i}}
\providecommand{\soi}{{(i)}}
\providecommand{\so}{{(\cdot)}}

\providecommand{\dist}[1]{\mathrm{dist}\left(#1\right)}

\providecommand{\Ia}{\mathds{1}_{\mathcal{A}}}
\providecommand{\Iai}{\mathds{1}_{\mathcal{A}_\smi}}

\newcommand{\prox}{\mathrm{prox}}

\newcommand{\IF}{\mathrm{IF}}
\newcommand{\DFBETA}{\mathrm{DFBETA}}

\newcommand{\D}{\mathrm{D}}
\newcommand{\enm}[1]{\llbracket #1\rrbracket}

\DeclareMathOperator*{\argmin}{arg\,min}

\theoremstyle{plain}

\theoremstyle{plain}
\newtheorem{theorem}{Theorem}[section]
\newtheorem{proposition}[theorem]{Proposition}
\newtheorem{lemma}[theorem]{Lemma}

\theoremstyle{plain}

\newtheorem{assumption}[theorem]{Assumption}
\theoremstyle{plain}
\newtheorem{remark}[theorem]{Remark}
\theoremstyle{plain}
\newtheorem{conjecture}[theorem]{Conjecture}

\usepackage{authblk}

\title{Influence Diagnostics in  High-dimensional M-estimation: \\
Precise Asymptotics }

\author{Hugo Cui}

\affil{\normalsize  Université Paris-Saclay, CNRS, Laboratoire de mathématiques d’Orsay, 91405, Orsay, France
}

\date{}

\begin{document}

\maketitle

\begin{abstract}
   The impact of a given training point on a statistical model can be  measured through its \textit{leave-one-out influence} on the model parameters, which quantifies how its removal from the training set affects the learned weights. 
   For convex M-estimation under Gaussian design, in the high-dimensional limit $n\asymp d$, we show that the empirical distribution of influences across training points concentrates around a deterministic measure which we sharply characterize. This characterization suggests that influential samples tend to lie on average close to the decision boundary, making contact with a standard data selection heuristic in active learning. 
\end{abstract}

\section*{Introduction}

How does a given training sample ultimately shape the predictions of a statistical model? Given a dataset $\mathcal{D}=\{x_i, y_i\}_{i\in\enm{n}}$ with $n$ covariate/label pairs, the relevance of the $i-$th sample $(x_i, y_i)$ may be most naturally captured by its \textit{leave-one-out influence} \citep{hampel1974influence, cook1979influential, cook1982residuals, cook1980characterizations, hampel86robust}, defined as the difference in model parameters, predictions or error when training on  $\mathcal{D}\setminus (x_i,y_i)$ (from which the sample was deleted), rather than $\mathcal{D}$. The leave-one-out influence thus isolates the effect of a sample on the trained model, affording a highly interpretable measure of importance. %Perhaps then unsurprisingly, influences constitute a fundamental explainability metric at the confluence of several fields. 
The present manuscript explores the natural question of the \textit{distribution} of influences : \textit{what fraction of the training samples have high influence? In which regions of input space are influential samples most likely to be found?}\\

%In applied machine learning, influence metrics are leveraged to explicate black-box neural network predictions  \citep{koh2017understanding, barshan2020relatif} across a broad section of modern deep learning, from diffusion models \citep{mlodozeniec2025influence} to large language models \citep{schioppa2022scaling}.\\
These questions are well-understood in the low-dimensional, large sample limit $n\to\infty, d=O(1)$, in which the influence of a given data point converges to a deterministic function (the \textit{influence function} \citep{hampel1974influence, hampel86robust, AvellaMedina2017InfluenceFF}), that depends solely on the considered sample.  In contrast, many modern machine learning models operate in a \textit{high-dimensional regime} where the ambient dimension $d$ is often comparable to, or larger than, the number of samples $n$. Such limits present unique statistical challenges \citep{johnstone2009statistical, maleki2026high, sur2019modern}. In particular, the standard convergence of leave-one-out influences breaks down, as the influence of an individual sample develops intricate statistical dependencies on \textit{all the other training samples} \citep{karoui2013asymptotic, el2018impact}, posing analytical hurdles to the analysis of influence distributions in high-dimensional learning problems.\looseness=-1

The present work studies the statistics of leave-one-out influences in the context of the classical problem of high-dimensional M-estimation with convex loss, under isotropic Gaussian design, building on the seminal characterization of \citep{karoui2013asymptotic, el2018impact} of leave-one-out estimators. In this framework, we precisely characterize the distribution of norms of \textit{difference-in-betas} (DFBETA) leave-one-out influences, defined as the change in estimator induced by the removal of a training sample, borrowing a metric introduced by \cite{belsley1980regression} in the context of linear regression. More precisely, our \textbf{main contributions} are the following:
\begin{itemize}[noitemsep]
    \item In the high-dimensional regime $d\asymp n$, we show that the empirical distribution of DFBETA norm influences across the training set concentrates around a deterministic measure, which we precisely characterize.
    \item An analysis of the limiting measure suggests that the average influence of a sample decreases with its distance to the decision boundary, lending support to well-known data-selection heuristics in active learning \citep{tong2001support, roth2006margin, wang2014new}. We further qualitatively validate these findings in real-data experiments.
\end{itemize}

\subsection*{Related works}

\paragraph{Influence functions ---}  Leave-one-out deletion diagnostics find their roots in the robust statistics literature, where they were first developed in the seminal works of \cite{cook1979influential, belsley1980regression, cook1980characterizations, cook1982residuals} for outlier detection in linear regression, and later extended by \cite{johnson1985influence} for logistic regression. Asymptotically as $n\to \infty$, these sample leave-one-out estimators converge to \textit{influence functions} \citep{ hampel86robust, tsiatis2006semiparametric}, which formalize the notion of a first derivative, when perturbing the population data distribution along the direction of the probed sample \citep{hampel1974influence, chatterjee1986influential}. This formalism in particular clarifies the connection with 
re-weighting diagnostics  \citep{pregibon1981logistic, cook1986assessment, thomas1990assessing}, where a sample is infinitesimally downweighted in the empirical risk, rather than deleted entirely. These diagnostics were later adapted to construct computationally efficient approximations of leave-one-out influences in deep learning models \citep{koh2017understanding}. The foundations of the analysis of leave-one-out influences in \textit{high-dimensional} regimes $d\asymp n$, on the other hand, were cast in the seminal works of \cite{el2013robust,karoui2013asymptotic, el2018impact}. The present manuscript elaborates upon these results (recalled in Subsection \ref{subsec:LOO}) to reach a characterization of the empirical distribution of leave-one-out influences.

\paragraph{High-dimensional M-estimation ---} A now rather mature body of works has been devoted to the study of M-estimators, namely Empirical Risk Minimization (ERM) problems over linear models, in the proportional high-dimensional regime $n\asymp d$ \citep{ karoui2013asymptotic, el2018impact, donoho2016high, thrampoulidis2018precise, dobriban2018high, sur2019modern, bellec2025observable}, making it a particularly natural testbed to investigate the high-dimensional behavior of leave-one-out influences. Further even, from a technical standpoint, leave-one-out procedures in fact constitute one of the key \textit{proof methods} in the analysis of such problems. Formalized in particular in the works of \cite{karoui2013asymptotic, el2013robust, el2018impact}, leave-one-out techniques are formally leveraged to disentangle the statistical dependency of the estimator on any individual data point. These studies hence naturally lay the groundwork for the analysis of sample influences. \\

The remainder of the manuscript is organized as follows. Section \ref{sec:asymptotics} expounds our main technical contribution, namely a sharp characterization of the limiting influence distribution in the high-dimensional regime $n\asymp d$, for convex M-estimators under Gaussian design. The implications of these results are subsequently explored heuristically in section \ref{sec:selection}, where the limiting  distribution is marginalized to determine how the influence depends, on average, on simpler geometric descriptors, such as the distance to the decision boundary, uncovering in which regions influential data points tend to lie. In doing so, we underline connections with well-known data selection strategies in active learning, which seek to identify informative samples.

\section{High-dimensional asymptotics of leave-one-out influences}\label{sec:asymptotics}

\subsection{Setting}

We consider throughout the manuscript the problem of ERM with a linear model on a training set $\mathcal{D}=\{x_i,y_i\}_{i\in\enm{n}}$, with $n$ data points. Following a long stream of works in high-dimensional M-estimation  \citep{ el2013robust, donoho2016high, thrampoulidis2018precise, sur2019modern}, we assume that the samples follow a $d-$dimensional Gaussian distribution $x_i\sim \mathcal{N}(0_d,I_d)$, while the corresponding labels are given by
\begin{align}
    \label{eq:label_gen}y_i=\phi(\inprod{x_i, \beta}),
\end{align}
for a ground-truth vector $\beta\in \mathbb{S}^{d-1}$, and a link function $\phi:\R \to \R$. For simplicity, we assume that the labels depend deterministically on the covariate, and defer a discussion of stochastic or noisy settings to Appendix \ref{app:extension}. 
Given the training set $\mathcal{D}$, we now consider the \textit{ridge-regularized M-estimator}
\begin{align}
    \hat{w}=\argmin_{w\in\R^d} \frac{1}{n}\sum\limits_{i\in\enm{n}} L(\inprod{x_i, w}, y_i) +\frac{\lambda}{2}\norm{w}^2,\label{eq:full_ERM}
\end{align}
where $L:\R^2\to \R_+$ is a loss function convex in its first variable, and $\lambda >0$ designates the regularization strength. For conciseness of notations, we absorb in the following the link function $\phi(\cdot)$ into the definition of the loss by introducing the effective loss $\ell(\cdot, \cdot)=L(\cdot, \phi(\cdot))$. The effective loss $\ell$ is in particular assumed to admit bounded derivatives up to order four in its first variable. We refer readers to Appendix \ref{app:proof} for a more exhaustive list of all technical assumptions.\\ %The performance of the learned model may then be quantified by the \textit{test error}
%$\egen[\hat{w}]=\Eb{x,y}{L(\inprod{x, \hat{w}},y)}$,
%which measures the average discrepancy between the model prediction $\inprod{x,\hat{w}}$ and the true label $y$ on a fresh test sample $(x,y)$, according to a cost metric $L:\R^2\to \R_+$.  In the following, when considering binary classification, the test error $\egen$ \eqref{eq:test_error} is understood to be the \textit{misclassification error}, namely the probability that a fresh test sample is wrongly classified by the trained model. Formally, $L(z,y)=\mathds{1}[\sign(z)\ne y]$, and correspondingly $\mathcal{E}(q,m)=\sfrac{1}{\pi}\arccos{\sfrac{m}{\sqrt{q}}}$. When considering regression tasks, we instead adopt the standard \textit{mean squared error} metric $L(z,y)=(z-y)^2$, for which $\mathcal{E}(q,m)=1+q-2m$. We refer the interested reader to subsection \ref{subsec:notations} for a detailed list of technical assumptions on the loss function $\ell(\cdot, \cdot)$ \eqref{eq:full_ERM} and the test error metric $\mathcal{E}(\cdot, \cdot)$ \eqref{eq:test_error}, of which we highlighted above only the most salient few.
%\\

For any index $i\in\enm{n}$, one may similarly define the \textit{leave-i-out} estimator $\hat{w}_\soi$, which minimizes the empirical risk 
\begin{align}\label{eq:soi_ERM}
    \hat{w}_\soi=\argmin_{w\in\R^d} \frac{1}{n-1}\sum\limits_{j\in\enm{n}\smi } \ell(\inprod{x_j, w}, \inprod{x_j, \beta}) +\frac{\lambda}{2}\norm{w}^2,
\end{align}
on the dataset $\mathcal{D}\setminus (x_i, y_i)$ from which the $i-$th sample $(x_i, y_i)$ has been deleted. Observe that the sum of individual losses in the leave-$i$-out ERM \eqref{eq:soi_ERM} now needs to be normalized by $\sfrac{1}{n-1}$, instead of $\sfrac{1}{n}$ as in the full ERM \eqref{eq:full_ERM}, so as to reflect the actual number of training samples, after deletion of $(x_i, y_i)$. The influence of a given data point on the ERM estimator \eqref{eq:ERM} can then be naturally measured by the norm of DFBETA influences, borrowing a name initially introduced for linear regression by \cite{belsley1980regression}. For any $i\in\enm{n}$, we define the DFBETA (norm) influence of the $i-$th sample as
\begin{align}\label{eq:def_DFBETA}
    \DFBETA_i=n\norm{\hat{w}-\hat{w}_\soi}^2.
\end{align}
DFBETAs thus measure the magnitude of the change \textit{in estimator}, in Euclidean distance, following the deletion of a sample. 
Intuitively, a training sample with large DFBETA value has a large effect in determining the learned estimator, while a training sample with small DFBETA plays a lesser role.

\subsection{Leave-one-out approximations}
\label{subsec:LOO}
In order to analyze the influence metric \eqref{eq:def_DFBETA}, it is hence essential to first develop a fine understanding of the joint statistics of the full ERM estimator $\hat{w}$ \eqref{eq:full_ERM}, and the leave-$i$-out estimate $\hat{w}_\soi$ \eqref{eq:soi_ERM}, for any index $i\in\enm{n}$. A closed-form expression of the corresponding discrepancy  $\hat{w}-\hat{w}_\soi$ is known since the works of \cite{karoui2013asymptotic, el2013robust, el2018impact, rad2019scalable}, with earlier related results emanating from the works of \cite{hampel1974influence, pregibon1981logistic, cook1982residuals}. In the notations of the present work, and incorporating minor additional technical results developed in Appendix \ref{app:influence}, this expression can be compactly written as
\begin{align}\label{eq:classical_LOO}
    \hat{w}=\hat{w}_\soi -\frac{1}{n}\partial \ell(\inprod{\hat{w}, x_i} ,\inprod{\beta, x_i}) H_\soi^{-1} x_i+ O_{L_k}\left(\frac{\polylog(n)}{n}\right).
\end{align}
In the above display, the $O_{L_k}\left(\sfrac{\polylog(n)}{n}\right)$ notation hides a random vector whose norm possesses a $k-$th moment bounded by $\left(\sfrac{\polylog(n)}{n}\right)^k$, for any integer $k$. We introduced the \textit{leave-$i-$out} \textit{Hessian} matrix
\begin{align} \label{eq:LOOHessian}
    H_\soi=\frac{1}{n-1}\sum\limits_{j\in\enm{n}\smi} \partial^2\ell\left(\inprod{\hat{w}_{(i)}, x_j},\inprod{ \beta, x_j}\right)x_jx_j^\top +\lambda I_d,
\end{align}
where $\partial$ denotes the derivative of the loss $\ell(\cdot, \cdot)$ with respect to its first argument.
In simple terms, the leave-one-out relation  \eqref{eq:classical_LOO} captures the effect of adding (or equivalently, removing) a data point $x_i$ to the training set. This effect is proportional to the gradient $\partial \ell(\inprod{\hat{w}, x_i}, \inprod{\beta, x_i})x_i$, and mediated by the Hessian matrix $H_\soi$, which captures the local geometry of the loss landscape around the estimate $\hat{w}_{(i)}$. Perturbation aligning with flatter directions (namely eigendirections of largest eigenvalue of the inverse Hessian $H_{(i)}^{-1}$) are thus conducive to the largest change in estimate $\hat{w}-\hat{w}_\soi$. 

\subsection{Distribution of DFBETA influences}

We are now in a position to report the main technical contribution of the present work, namely a characterization of the  limiting distribution of leave-one-out influences $\DFBETA_i$ \eqref{eq:def_DFBETA} in the high-dimensional regime $n\asymp d$, for convex M-estimation under Gaussian design \eqref{eq:full_ERM}.  
%The next Proposition answers the following questions: what is the histogram of DFBETA influences?

\begin{proposition}[Concentration of the empirical DFBETA distribution] \label{prop:DFBETA}
    Let $\mathcal{D}=\{x_i, y_i\}_{i\in\enm{n}}$ be a collection of $n$ data points sampled identically and independently as \eqref{eq:label_gen}, and $\{\DFBETA_i\}_{i\in\enm{n}}$ be the corresponding $\DFBETA$ influences \eqref{eq:def_DFBETA} associated with the ERM \eqref{eq:full_ERM}. In the asymptotic limit $n,d\to \infty$ with fixed ratio $\alpha:=\sfrac{n}{d}$, the empirical distribution
    \begin{align}
        \hat{\nu}_D=\frac{1}{n}\sum\limits_{i\in\enm{n}}\delta_{\DFBETA_i}
    \end{align}
concentrates around the deterministic measure
\begin{align}
    \label{eq:pushforward} \nu_D=\varphi_D~\sharp~\mathcal{N}(0_2, Q), 
\end{align}
in the sense that for any twice-differentiable bounded test function $f: \R\to \R$ with bounded derivatives, it holds that $$\Eb{z\sim \hat{\nu}_D}{f(z)}=\Eb{z\sim \nu_D}{f(z)}+O_{L_2}\left(\frac{\polylog(n)}{n^{\frac{1}{4}}}\right).$$ Namely, empirical averages of functions of the DFBETAs concentrate at a $\sfrac{\polylog(n)}{n^{\frac{1}{4}}}$ rate.
In \eqref{eq:pushforward}, the map $\varphi_D:\R^2\to\R_+$ is defined by
\begin{align}\label{eq:phi_D}
    \varphi_D(g)=\left(\frac{\prox_{V^{(1)}\ell(\cdot, g_2)}(g_1)-g_1}{V^{(1)}}\right)^2 V^{(2)},
\end{align}
while the covariance $Q$ is given by
\begin{align}\label{eq:Q}
    Q=\begin{bmatrix}
        q&m\\
        m&1
    \end{bmatrix},
\end{align}
with $q,m$  satisfying the self-consistent equations \citep{karoui2013asymptotic, el2018impact, thrampoulidis2018precise, cui2026asymptotic}
\begin{align}\label{eq:charac_Q}
        &  
        q=-\frac{1}{\lambda}\Ea{\partial\ell(r,g_2)r}
        +O\left(\frac{\polylog(n)}{n^{\frac{1}{4}}}\right),\\
    &      
       m =-\frac{1}{\lambda}\Ea{\partial\ell(r,g_2)g_2}
      + O\left(\frac{\polylog(n)}{n^{\frac{1}{4}}}\right).\end{align}
 In the above displays, the summary statistics $V^{(k)}$ are related through
\begin{align}
    &V^{(k)}=\frac{-1}{2\pi i}\oint_\Gamma \frac{1}{z^k} \mathcal{V}(z) dz +O\left(\polylog(n)e^{-\sfrac{n}{2}}\right), \label{eq:Vk}
\end{align}
to the Stieljes transform $\mathcal{V}:\C\to \C$. In \eqref{eq:Vk}, the contour $\Gamma$ is chosen as the boundary of the $\sfrac{\lambda}{2}-$neighborhood of the interval $J=\left[
    \lambda, \lambda +\norm{\partial^2\ell}_\infty\left(2+\sqrt{\sfrac{1}{\alpha}}\right)^2
    \right]$.  The Stieltjes transform $\mathcal{V}(\cdot)$ further satisfies for all $z\in \Gamma$ with $ \Im{z}\ne 0$ the self-consistent equation 
\begin{align}\label{eq:charac_Vz}
   \frac{1}{\alpha}-(\lambda-z)\mathcal{V}(z)-\Ea{\frac{\partial^2\ell\left(r,g_2\right)\mathcal{V}(z)}{1+\partial^2\ell\left(r,g_2\right)\mathcal{V}(z)}}=O\left(\frac{\polylog(n)}{n^{\frac{1}{4}}}\right).
\end{align}
Across the above displays, we employed the shorthand $r=\prox_{V^{(1)}\ell(\cdot, g_2)}(g_1)$,  while expectations bears over $g\sim\mathcal{N}(0_2, Q).$
\end{proposition}

In words, the empirical distribution of the DFBETA influences of the $n$ training samples can be concisely expressed in the $n\asymp d$ limit as the pushforward of a two-dimensional Gaussian density with zero mean and covariance $Q$ \eqref{eq:Q} through a non-linear map $\varphi_D$ \eqref{eq:phi_D}. The entries $m,q$ \eqref{eq:Q} of the covariance $Q$ capture geometric descriptors $\Ea{\inprod{\hat{w}, \beta}}, \Ea{\norm{\hat{w}}^2}$ of the ERM minimizer $\hat{w}$ \eqref{eq:full_ERM}, and are consequently often referred to under the umbrella of \textit{summary statistics} in the exact asymptotics literature. The statistics $V^{(k)}$ capture on the other hand the $k-$th tracial moments $\sfrac{1}{n}\Ea{\tr[H_{(i)}^{-k}]}$ of the Hessian \eqref{eq:LOOHessian}. We stress that equations \eqref{eq:charac_Q} are classical, and can be found for instance reported in the seminal works of \cite{karoui2013asymptotic, el2018impact, donoho2016high, thrampoulidis2018precise}. The self-consistent characterization of the Stieltjes transform $\mathcal{V}(z)$ \eqref{eq:charac_Vz}, on the other hand, is to our awareness a more recent result, derived in greater generality in \cite{asgari2025local} in a stronger form, where it also comes complete with a proof of uniqueness. The proof of Proposition  \ref{prop:DFBETA}, presented in Appendix \ref{app:proof} builds upon those foundations, and additionally establishes the concentration of the tracial moments $V^{(k)}$ and the DFBETA distribution $\hat{\nu}_D$. \\

\begin{wrapfigure}{r}{0.45\textwidth}
    \centering
    \includegraphics[width=0.62\linewidth]{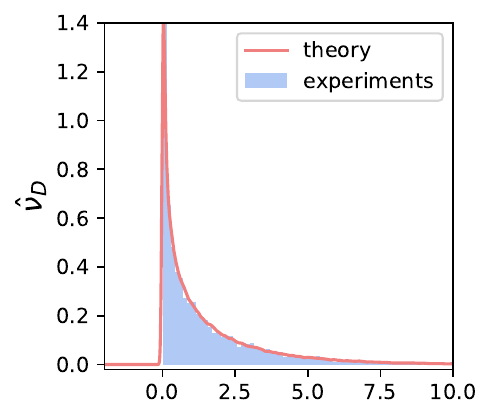}
   
    \caption{Empirical distribution $\hat{\nu}_{D}$ of the leave-one-out DFBETAs  \eqref{eq:def_DFBETA} across the training set, for binary classification ($y=\mathrm{sign}(\inprod{\beta,x})$), with logistic loss $\ell(z,y)=\ln(1+\exp(-yz))$, at $\alpha=2, \lambda=0.05$. The blue histogram represents numerical experiments on Gaussian data, in dimension $d=4000$. The red solid line indicates the deterministic equivalent $\varphi_{\D}\sharp \mathcal{N}(0_2, Q)$  characterized in Proposition \ref{prop:DFBETA}.  }
    \label{fig:histograms}
\end{wrapfigure}

The two-dimensional Gaussian distribution $\mathcal{N}(0_2, Q)$ which underlies the limiting influence distribution $\nu_D$ in fact corresponds to the weak limit of the joint law of the random variables $\inprod{\hat{w}_\soi, x_i},\inprod{\beta, x_i}$ \citep{karoui2013asymptotic, el2018impact, cui2026asymptotic}. These random variables are simple geometric descriptors that reflect how the deleted sample $x_i$ aligns within the geometry of the leave-$i-$out problem \eqref{eq:soi_ERM}. This insight allows a finer investigation of the spatial dependencies of the influence. For instance, the conditional distribution of $\DFBETA_i |\inprod{\hat{w}_\soi, x_i} $ can be heuristically assessed by the pushforward through $\varphi_D$ of the Gaussian $\mathcal{N}(0_2, Q)$ \textit{conditioned on its first entry.} These ideas are explored in further detail in Section \ref{sec:selection}, in the context of the problem of active learning, in which one seeks to develop simple agnostic predictors of the influence (as a proxy for informativeness) of a given sample.\looseness=-1\\

From a practical standpoint, the set of self-consistent equations \eqref{eq:charac_Q} and \eqref{eq:charac_Vz} (without the error terms) can in general be solved numerically, and the limiting summary statistics $Q, V^{(k)}$ thereby evaluated. In simple settings, for instance for the square loss $\ell(z,y)=\sfrac{1}{2}(z-y)^2$, the functional equation \eqref{eq:charac_Vz} starkly simplify, yielding closed-form expressions for the resolvents $\mathcal{V}(\cdot)$, and ultimately for $\varphi_D~\sharp~\mathcal{N}(0_2, Q)$. The following remark formalizes this observation, and shows that for the square loss, the limiting distribution of DFBETAs \eqref{eq:def_DFBETA} can be concisely expressed as a rescaled $\chi^2$ distribution.

\begin{remark}\label{remark:Ridge}
    For the square loss $\ell(z,y)=\sfrac{1}{2}(z-y)^2$, assuming a linear model $y_i=\inprod{x_i,\beta}+\epsilon_i$ \eqref{eq:label_gen} where $\epsilon_i\sim\mathcal{N}(0,\delta^2)$ are independent Gaussian additive label noises, an evaluation of the equations within Conjecture \ref{Conj:extension} (which extends Proposition \ref{prop:DFBETA} for noisy settings) yields
    \begin{align}
    \nu_D=\frac{V^{(2)}}{\left(1+ V^{(1)}\right)^2}\left(q+1-2m+\delta^2\right)\chi^2_1.
\end{align}
Note that the term in parentheses corresponds to the mean squared error, suggesting that models achieving smaller test error are also characterized by a narrower distribution of influences.
\end{remark}

It is important to stress that the square loss does not in principle satisfy the assumption of bounded loss \ref{ass:loss} under which Proposition \ref{prop:DFBETA} is stated. However, we note that it can be approximated arbitrarily tightly by a sequence of losses that do obey the assumptions. We hence expect that such restrictions are purely technical in nature, and are amenable to being relaxed by standard mollification arguments.\\ 

Fig.\,\ref{fig:histograms} contrasts the limiting influence distribution $\varphi_D\sharp \mathcal{N}(0_2, Q)$, characterized in Proposition \ref{prop:DFBETA}, with numerical experiments in large but finite dimension $d=4000$, in the problem of binary classification with logistic loss, displaying overall good agreement. The histogram reveals a large density of samples of low-to-moderate influence (small DFBETA), and a rapidly decaying tail of higher-influence samples. It is then natural to wonder \textit{where}, in input space, these influential samples may most likely be found. The following section explores this question, leveraging the findings of Proposition \ref{prop:DFBETA}.

\begin{figure}
    \centering
    \includegraphics[width=0.22\linewidth]{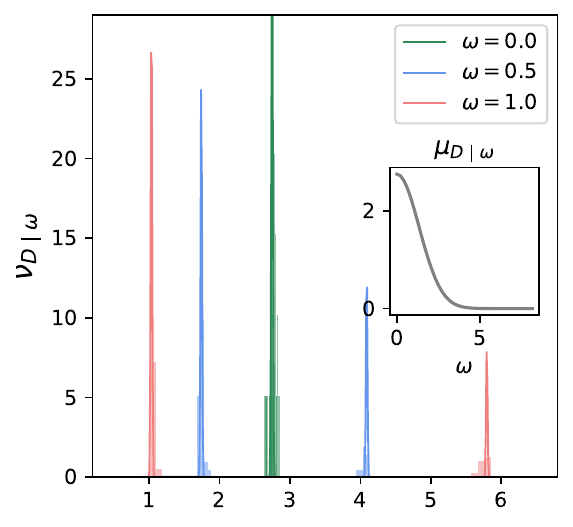}
   \includegraphics[width=0.27\linewidth]{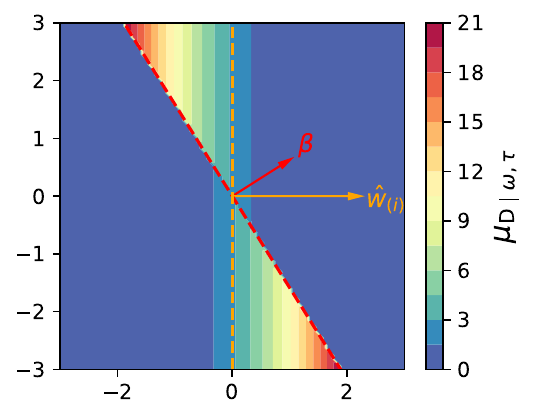}
    \caption{Logistic regression $\ell(z,y)=\ln(1+\exp{-yz}) $ with $\lambda=0.05, \alpha=2$, binary classification task. (\textbf{left}) Marginal distribution $\nu_{D\mid \omega}$ of the DFBETA influence conditional on the margin $\inprod{x_i,\hat{w}_\soi}=\omega$. Lines are obtained from the conditional limiting density $\varphi_D~\sharp ~\mathcal{N}(0_2, Q\mid \omega)$ characterized in Proposition \ref{prop:DFBETA}. Inset: mean influence $\mu_{D\mid \omega}$ conditional on the margin being $\omega$, as a function of $\omega$. Histograms represent numerical simulations in dimension $d=3000$, grouping samples whose margin falls within the interval $\omega \pm 0.02$.  (\textbf{right}) Mean $\DFBETA$ influence \eqref{eq:def_DFBETA} $\mu_{D \mid \omega, \tau}$ conditioned on the margin conditions $\inprod{x_i,\hat{w}_\soi}=\omega$ and $\inprod{x_i,\beta}=\tau$, evaluated from Proposition \ref{prop:DFBETA}.}
    \label{fig:map_IF_2}
\end{figure}

\section{Consequences for active learning} \label{sec:selection}
%Having characterized in Section \ref{sec:asymptotics} the statistics of leave-one-out influences \eqref{eq:def_IF} in high-dimensional regimes, we now find ourselves in a position to explore the consequences of Theorem \ref{thm:main_IF} for \textit{active learning} tasks, where one seeks to identify influential samples from simple agnostic criteria.\\

Leave-one-out influences provide interpretable measures of the importance of a data point for a learning task. This question also lies at the heart of the fields of \textit{active learning} and \textit{subsampling} : given a pool of available data points, limitations in annotation or compute budget often make it necessary to select/subsample \textit{only a fraction }$\gamma\in(0,1)$ of samples to use for training \citep{settles2009active}. In such cases, it is hence seminal to identify and select the most important samples.
Sample influences can in such contexts serve as natural proxies for the informativeness of a training point. The problem of data selection has been classically explored in large-sample regimes \citep{wang2018optimal, ma2015statistical}. \cite{ting2018optimal} show that choosing subsampling probabilities to be proportional to the norm of (population) DFBETA influence functions satisfies a notion of A-optimality \citep{kiefer1959optimum}. Adjoining ideas were reprised in a number of subsequent works. Let us mention for instance \cite{wang2020less}, who propose a subsampling scheme retaining only samples with whose removal increases the test error. A stream of recent works \citep{cui2021large, kolossov2023towards, sorscher2022beyond, dohmatob2025less, askari2025improving, cui2026asymptotic, mai2022effectiveness} study active binary classification in the $n\asymp
d$ limit, under Gaussian design, impulsed by the seminal works of \cite{kinzel1990improving, seung1992query}. These works all consider a specific margin-based selection heuristic \citep{tong2001support, roth2006margin, wang2014new}, where samples lying closest to the estimated decision boundary are selected. In the following paragraphs, we illustrate how the results of Proposition \ref{prop:DFBETA} suggest that the margin of a sample is, in fact, correlated with its leave-one-out influence, lending further support to this heuristic.

\paragraph{Margin and influence --- }
Consider an iterative active learning procedure, where samples
are selected one by one to be labeled and used for training. Suppose the availability of an estimator, learned from hitherto selected
samples. Which sample, among the remaining unlabeled points, should be selected next ? The question essentially boils down to understanding where, in input space, the most influential samples are most likely to be found. A natural heuristic in the case of linear classification is to consider the distance of a sample to the current decision boundary. The correlation between the influence and margin of a sample can be probed by studying the law of $\DFBETA_i $ conditional on the sample $x_i$ sitting at a prescribed margin $\inprod{\hat{w}_\soi, x_i}=\omega$. Proposition \ref{prop:DFBETA} and the discussion below suggest that in the high-dimensional regime $n\asymp d$, this conditional distribution may be approximated by the pushforward $\nu_{D\mid \omega}=\varphi_D~\sharp ~\mathcal{N}(0_2, Q\mid \omega)$, where the shorthand $\mathcal{N}(0_2, Q\mid \omega)$ denotes the law of $g\sim \mathcal{N}(0_2, Q)$ conditional on $g_1=\omega$. Fig.\,\ref{fig:map_IF_2}  illustrates the conditional limiting influence density $\nu_{D\mid \omega}$ and its mean $\mu_{D \mid \omega}$ for binary classification with logistic regression. 
The average DFBETA influence decreases with the distance $\omega$ to the decision boundary (inset), lending support to the well-known active learning heuristic that samples with smallest margins are most influential. A finer study of the conditional distribution unveils a more nuanced picture. At a given margin $\omega$, the conditional distribution generically exhibits a bimodal structure, with a large atom at low influence, and a smaller atom at high influence. The separation between the two atoms (both in influence and density) becomes more marked as $\omega$ is increased. \\
%This conclusion needs on the other hand to be nuanced at small sample complexities, where the conditional influence distribution $\nu_{\IF\mid \omega}(\cdot)$ only weakly depends on the margin $\omega$ (Fig.\,\ref{fig:map_IF_005}), somewhat echoing the observations of \cite{hacohen2022active, hacohen2023select, sorscher2022beyond} that the small-margin selection policy might cease to be appropriate in data-scarce regimes.

\begin{figure}
    \centering
    \includegraphics[width=0.29\linewidth]{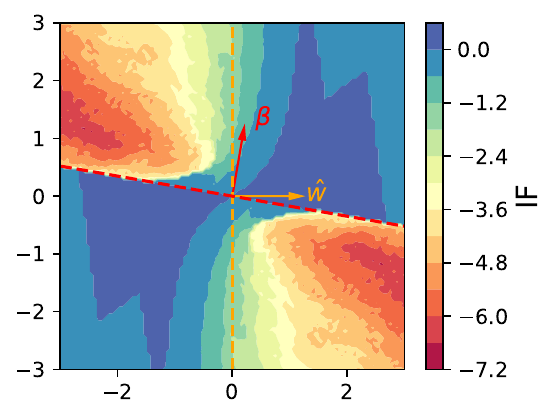}
    \includegraphics[width=0.29\linewidth]{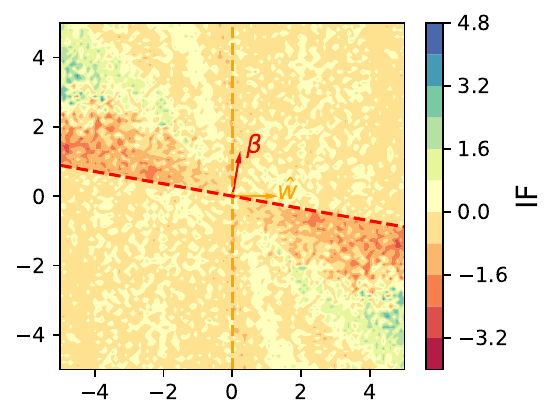}
    \caption{Test-error influence $\IF$  for the odd versus even classification task on MNIST digits \citep{lecun1998gradient} (\textbf{left}), and for the pneumonia diagnosis task on chest X-rays images \citep{kermany2018identifying} (\textbf{right}). A $3-$layer, ReLU-activated neural network feature map was trained on a first held-out dataset, and the population readout $\beta$ estimated from retraining on another held-out dataset. The weights $\hat{w}$ are estimated from the training set, and the color map represent the difference in test error following the inclusion of a synthesized data point chosen in the $(\beta, \hat{w})$ plane, in the likeness of Fig.\,\ref{fig:map_IF_2}. All experimental details are collected in Appendix \ref{app:numerics}.  }
    \label{fig:Real_data}
\end{figure}

A more fine-grained picture of the spatial distribution of influential samples may be obtained by further conditioning on the ground-truth margin $\inprod{x_i,\beta}$, although we note that for practical purposes the latter is generically unavailable to the statistician. Formally, this implies examining the law of DFBETA influences conditional on the sample $x_i$ sitting at relative margins $\inprod{\hat{w}_\soi, x_i}=\omega, \inprod{\beta, x_i}=\tau$. Again, Proposition \ref{prop:DFBETA} suggests that $\nu_{D\mid \omega, \tau }$ should be described by an atom at $\mu_{D\mid \omega, \tau}=\varphi_D(\omega, \tau)$. The corresponding value is illustrated as a color map in Fig.\,\ref{fig:map_IF_2} (right), in the $(\omega, \tau)$ plane. In agreement with intuition, at a given margin, the higher-influence samples are found in the region misclassified by the model, while samples correctly classified display significantly smaller influence.\looseness=-1

\paragraph{Real data ---} Having charted out the spatial distribution of influential samples in the Gaussian case, we now examine whether this picture still holds at least qualitatively for learning tasks on \textit{real data}, and for other influence metrics. We explore this question with the following experiment, described in further detail in Appendix \ref{app:numerics}, the results of which we report in Fig.\,\ref{fig:Real_data}, for the metric of \textit{test error} influence $\IF$, defined as the difference in misclassification error induced by the inclusion of a sample, times $n$. The MNIST \citep{lecun1998gradient} (left) and chest X-ray scans \citep{kermany2018identifying} (right) datasets were split into four disjoint sets $\mathcal{D}_{\mathrm{NN}}, \mathcal{D}_\beta, \mathcal{D}_{\rm train}, \mathcal{D}_{\rm test}$. A feature map $\mathrm{NN}(\cdot)$ was first obtained from $\mathcal{D}_{\mathrm{NN}}$, by training a 3-layer fully-connected neural network with ReLU activation using the Adam optimizer \citep{kingma2014adam}, respectively on the even/odd digit classification task (for the MNIST dataset) and the pneumonia diagnosis task (for the chest X-rays dataset). The readout weights $\beta$ were then retrained on a large dataset $\mathcal{D}_\beta$, to estimate a good separator $\beta$ of the $\mathrm{NN}(\cdot)$ features. A smaller training set $\mathcal{D}_{\rm train}$ instead yield a less accurate estimator $\hat{w}$. What is then the most influential sample to add to $\mathcal{D}_{\rm train}$, so that the test error on $\mathcal{D}_{\rm test}$ is most reduced ? For every point $z$ in the plane spanned by $(\beta, \hat{w})$, we synthesize the sample $z$, assume the oracle label $y=\mathrm{sign}(\inprod{\beta,z)}$ can be queried, and add the sample $(z,y)$ to $\mathcal{D}_{\rm train}$. The difference in test error $\IF$ induced by the inclusion of this synthesized data point is then evaluated, with negative values signaling that the addition of the sample helps achieve better accuracy. The values are plotted in Fig.\,\ref{fig:Real_data} and reveal a spatial distribution qualitatively similar to that of DFBETA influences for Gaussian data (Fig.\,\ref{fig:map_IF_2}, right).\looseness=-1

\subsection*{Conclusions}

We study the leave-one-out influence of samples within a dataset in the high-dimensional regime $n\asymp d$. For convex M-estimation under Gaussian design, we show that the empirical distribution of DFBETA influences converges to a limiting distribution which we characterize. Building on these results, we heuristically probe the spatial distribution of the most influential samples, and provide evidence that the latter tend to lie close to the decision boundary, making contact with a standard data-selection heuristic in active learning.

\paragraph{Perspectives ---} The results reported in the present work exclusively concern convex M-estimation with linear models, namely single-layer neural networks. However, accruing empirical evidence \citep{basu2020influence, bae2022if} suggests a much more brittle behavior of leave-one-out influences in \textit{non-convex} landscapes, e.g. for neural network, which we believe makes the case for future investigations in such settings. Let us also mention the problem of \textit{subset influence} \citep{broderick2020automatic, hu2024most, fisher2022statistical, konrad2025testing}, which naturally extends the question of leave-\textit{one}-out influences to settings where \textit{groups} of samples, rather than individual points, are removed from the training set. We anticipate the analysis of subset influences in high-dimensional regimes to be made particularly challenging by the intricate statistical correlations within the carved-out subset, notably when the cardinal of the former is comparable to the sample size $n$. We believe this question delineates an interesting avenue for future research.

\newpage
\bibliography{references}
\bibliographystyle{plainnat}
\newpage 
\appendix

\section{Proof of Proposition \ref{prop:DFBETA}}
\label{app:proof}

\subsection{Assumptions}\label{subsec:notations}
We detail in the present appendix the proof of Proposition \ref{prop:DFBETA}.
 Before proceeding, we list in the next section all technical assumptions bearing on the ERM loss $\ell(\cdot, \cdot)$ \eqref{eq:full_ERM}, of which we only stressed the most salient ones in the main text.
%In the following, by a slight abuse of notation, we denote by $\ell:\R^2\to \R_+$ the \textit{effective} loss function that maps arguments $r,s$ to $\ell(r, \phi(s))$ in the notations of \eqref{eq:ERM}. In other words, the link function $\phi(\cdot)$ is absorbed into the definition of the loss function $\ell(\cdot, \cdot)$.

\begin{assumption}[Loss function $\ell$]\label{ass:loss}
Let $\ell:\R^2\to \R_+$ denote the effective loss function. We assume the following properties to hold:
    \begin{enumerate}[label=(A.\arabic*)]
        \item $\ell$ is everywhere four times differentiable in its first argument. We denote by $\partial$ the derivative with respect to the latter. The first four derivatives are furthermore bounded by $O(\polylog(n))$.
         \item  $\partial^2\ell$ is everywhere positive. 
        \item The reduced function $\tilde{\ell}: y\to \ell(0,y)$ is bounded by a constant $\sfrac{\lambda M^2}{2}$.
        \item $\partial\partial_2\ell,\partial^2\partial_2\ell$ (where $\partial_2$ denotes the derivative with respect to the second variable) are well-defined, and bounded by $O(\polylog(n))$.
        %\item The first derivative $\partial \ell$ satisfies $\Ea{\partial\ell(0, g_2)g_2}\ne 0$, when the average bears over $g_2\sim\mathcal{N}(0,1)$. This assumption implies in particular the existence of $c>0$ such that the sequence $\Ea{\norm{\hat{w}}^2}$ is lower-bounded by $c$, see Lemma \ref{lem:qnonzero}.
        %\item The set of discontinuity points of the first derivative $\partial\ell$ has zero measure. 
    \end{enumerate}
\end{assumption}

%\begin{remark}
 %   In words, Assumption A.6 can be interpreted as requiring that after a gradient step on the population loss, starting from a vanishing $0_d$ initialization, the weight iterates develop a non-zero alignment with the ground-truth $\beta$.
%\end{remark}

\begin{remark}
    While Assumption A.4 is not satisfied for a sign link function $\phi=\mathrm{sign}$ \eqref{eq:label_gen}, the latter can be approximated arbitrarily tightly as $n,d\to \infty$ by differentiable functions that do, following standard mollification arguments. We thus expect this restriction to be an artifact of the proof technique employed for Lemma \ref{lem:IaSz_concentrates}, rather than an intrinsic limitation of the problem.
\end{remark}

%\begin{remark}
%    Assumption A.6 allows to rule out pathological cases where $\hat{w}=0$ almost surely. It is satisfied for a large class of loss functions, and notably for the logistic loss $\ell(z,y)=\ln(1+\exp(-\mathrm{sign}(y)z)$. Assumption A.6 also admits a natural interpretation : suppose on solves the ERM leveraging gradient descent. The condition $\Ea{\partial\ell(0, g_2)g_2}\ne 0$ then corresponds to requiring that the first gradient step from a vanishing initialization is, in expectation, aligned with the ground truth vector $\beta$.
%\end{remark}

%\begin{assumption}[noise distribution]
%    The stochastic noises $\{\epsilon_i\}_{i\in\enm{n}}$ are drawn i.i.d from a probability distribution with all moments finite. \hc{We omitted the noise below. Decide if to add it or not. }
%\end{assumption}

%\begin{assumption}[Test error metric] \label{ass:metric}
 %Let $\mathcal{E}:\R\times \R_+\to \R_+$ be the test error metric \eqref{eq:test_error}. We assume that $\partial_1\mathcal{E},\partial_2\mathcal{E},\partial_1^2\mathcal{E},\partial_1^2\mathcal{E},\partial_1\partial_2\mathcal{E}$ are everywhere well-defined and bounded by $O(\polylog(n))$.
% \end{assumption}

\subsection{Notations}

We employ throughout the proof the following notations.
\begin{itemize}
    \item Given two sequences of random variables $a_n,b_n$ indexed by $n$, and an integer $k\in\mathbb{N}$, we denote $a_n=O_{L_k}(b_n)$ when $\Ea{\module{a_n}^k}=O\left( \Ea{\module{b_n}^k}\right)$. If left unspecified, the $O_{L_k}(\cdot)$ is implicitly understood to hold for all $k\in\mathbb{N}$. Similarly, for a sequence of random vectors $u_n$, a deterministic scalar sequence $v_n$, and an integer $k\in\mathbb{N}$, we denote $u_n=O_{L_k}(v_n)$ when $\Ea{\norm{u_n}^k}=O( \module{v_n}^k)$ holds for the Euclidean norm.
    \item Let $X$ be a $\mathcal{E}-$ valued random variable well-defined on an event $\mathcal{A}$, but not necessarily defined on $\mathcal{A}^c$. We define the \textit{truncated random variable} $\mathds{1}_{\mathcal{A}}X$ as the random variable taking value $(\mathds{1}_{\mathcal{A}}X)(\omega)=X(\omega)$ for all $\omega\in \mathcal{A}$, and $0_\mathcal{E}$ for $\omega \in \mathcal{A}^c$. 
    \item Given a matrix $X\in\R^{n\times d}$ and an index $i\in\enm{n}$, $X_\smi\in\R^{(n-1)\times d}$ denotes the matrix obtained from $X$ after removing its $i-$th row.
\end{itemize}

\subsubsection{Notations : leave-one-out estimators}

We now introduce the notations that will be employed in the remainder of the proof. We study the minimizer $\hat{w}$ of the empirical risk \eqref{eq:full_ERM}
\begin{align}\label{eq:ERM}
    \hat{w}=\underset{w}{\rm argmin} ~\hat{\mathcal{R}}(w), && \hat{\mathcal{R}}(w)=\frac{1}{n}\sum\limits_{j\in\enm{n}}\ell_{j}(\inprod{w,x_j})+\frac{\lambda}{2}\norm{w}_2^2,
\end{align}
using the shorthand $\ell_j(\cdot)=\ell(\cdot, \inprod{\beta, x_j})$. For any sample index $i\in \enm{n}$, we further introduce the \textit{leave-one out surrogate}
 \begin{align}\label{eq:wsmi}
    \hat{w}_{\smi}=\underset{w}{\rm argmin} ~\hat{\mathcal{R}}_{ \smi}(w), && \hat{\mathcal{R}}_{ \smi}(w)=\frac{1}{n}\sum\limits_{j\ne i }\ell_{j}(\inprod{w,x_j})+\frac{\lambda}{2}\norm{w}_2^2,
\end{align}
namely the minimizer of the empirical risk from which the $i-$th sample point was removed. 
We accordingly define the full and leave-one-out residuals
\begin{align}
    & r_j=\inprod{\hat{w}, x_j}, && r_{j,\smi}=\inprod{\hat{w}_{\smi}, x_j}.
\end{align}
Importantly, remark that the risk $\hat{\mathcal{R}}_\smi$ differs from \eqref{eq:soi_ERM} by its $\sfrac{1}{n}$, instead of $\sfrac{1}{n-1}$, normalization, and consequently $\hat{w}_\smi\ne \hat{w}_\soi$. A fine analysis of their discrepancy will be given in Lemma \ref{lemma:approx_norm} in Appendix \ref{app:influence}.
From the leave-on-out estimator $\hat{w}_\smi$, one can construct the \textit{surrogate estimator}
\begin{align}\label{eq:tildew}
    \tilde{w}_{i}=\hat{w}_{\smi}-\frac{1}{n}\partial\ell_{i}(\tilde{r}_{i,i})H_{\smi}^{-1}x_i,
\end{align}
which will prove a close approximation of the full estimator $\hat{w}$ \citep{karoui2013asymptotic, el2018impact}, a statement that will be made precise in Proposition \ref{prop: LOO_approxs}. In the above display, we introduced the \textit{leave-$i$
 out Hessian} $H_{\smi}$ as
 \begin{align}
    H_{\smi}=\frac{1}{n}\sum\limits_{j\ne i} \partial^2 \ell_{j}(r_{j,\smi})x_jx_j^\top +\lambda I_d
 \end{align}
 and denoted
 \begin{align}
     \tilde{r}_{i,i}=\prox_{\gamma_{i} \ell_{i}(\cdot)}(r_{i,\smi}), &&\gamma_{i}=\frac{1}{n} x_i^\top H_{\smi}^{-1}x_i.
 \end{align}
$H_{\smi}$ naturally yields a leave-$i-$out approximation of the full Hessian 
\begin{align}
    H=\frac{1}{n}\sum\limits_{j\in \enm{n}} \partial^2 \ell_{j}(r_{j})x_jx_j^\top +\lambda I_d.
\end{align}
  Finally, for $j\ne i$, we similarly define the surrogate residual
  \begin{align}
\tilde{r}_{j,i}=\inprod{\tilde{w}_{i}, x_j}.      
  \end{align}
For simplicity, we will also adopt the shorthand
\begin{align}
    \eta_i=-\frac{1}{n}\partial\ell_{i}(\tilde{r}_{i,i})H_{\smi}^{-1}x_i,
\end{align}
so that $\tilde{w}_{i}$ may be more compactly rewritten as $\tilde{w}_{i}=\hat{w}_{\smi}+\eta_i$. It will finally also prove convenient to introduce the surrogate Hessian
\begin{align}
     \tilde{H}_{i}=\frac{1}{n}\sum\limits_{j\in\enm{n}\smi} \partial^2 \ell_{j}(\tilde{r}_{j,i})x_jx_j^\top +\lambda  I_d.
\end{align}

\subsubsection{Notations : resolvent}
The Hessian matrix $H$ plays a central role in describing the local geometry of the empirical risk landscape \eqref{eq:full_ERM} around its minimizer.
Given $z\in \C\setminus \spec{H}$, we define the \textit{resolvent} $G(z)$ of $H$ as
\begin{align}
    G(z)=(H-z I_d)^{-1}.
\end{align}
We introduce the\textit{ Stieltjes transform}
\begin{align}
    \mathfrak{S}(z)=\frac{1}{n}\tr[G(z)], \label{eq:mathfrakS}
\end{align}
from which one can extract the tracial moments
\begin{align}
    \mathfrak{V}^{(k)}=\frac{1}{n}\tr[H^{-k}] =\frac{-1}{2\pi i}\oint_{\mathcal{C}} \frac{1}{z^k} \mathfrak{S}(z) dz \label{eq:mathfracV}
\end{align}
through the Cauchy integral formula for matrices, provided the contour $\mathcal{C}$ encloses $\spec{H}$, but not $0$. We will also make use of the leave-one-out equivalents of the above transforms. We accordingly define for $i\in\enm{n}$
\begin{align}
    &G_\smi(z)= (H_\smi-z I_d)^{-1}, && \tilde{G}_i(z)=(\tilde{H}_i-z I_d)^{-1},\\
    & \mathfrak{S}_\smi(z)=\frac{1}{n}\tr[G_\smi(z)], && \tilde{\mathfrak{S}}_i(z)=\frac{1}{n}\tr[\tilde{G}_i(z)].
\end{align}
Let us stress the following technical subtlety: the leave-one-out resolvents $G_\smi(\cdot), \tilde{G}_i(\cdot)$ are defined on the complex plane relative to the spectra of the Hessians $H_\smi, \tilde{H}_i$ respectively, which generically differ from that of $H$. Thus, they do not exactly share the same domain of definition as $G(\cdot)$. In what follows, it will be established that the summary statistics $\mathfrak{V}^{(k)}$ in fact concentrate in $L_2$ to their expectation. We will denote the latter by
\begin{align}
 V^{(k)}=\Ea{\mathfrak{V}^{(k)}}.
\end{align}
%It will also be useful to introduce the expectations $\mathcal{V}(\cdot)$ and $\Omega(\cdot)$ of the resolvent quantities $\mathfrak{S}(\cdot), \mathfrak{W}(\cdot)$. 

Particular care in the manipulation of the above resolvent quantities needs to be devoted to the fact that those are only well-defined away from $\spec{H}$. In particular, the use of the Cauchy integration formula implies a contour intersecting the real axis : the crossing point must then be ensured to avoid the spectrum. Thanks to the concentration of the spectral density of $H$ on a compact interval, this is possible with overwhelming probability.  Let us accordingly introduce the sequence of events  
\begin{align}
    \mathcal{A}=\left\{ \sfrac{\norm{X}^2}{n}<\left(2+\sqrt{\sfrac{1}{\alpha}}\right)^2\right\},
\end{align}
By standard non-asymptotic bounds for Wishart matrices (see e.g. \cite{vershynin2012introduction}), the probability of the complementary event vanishes exponentially
\begin{align}
    \P{\mathcal{A}^c}\le 2e^{-\sfrac{n}{2}}.
\end{align}
It will also prove instrumental to consider for $i\in\enm{n}$ the leave-one-out version of the event $\mathcal{A}_\smi$, similarly defined as 
\begin{align}
    \mathcal{A}_\smi=\left\{
   \sfrac{\norm{X_\smi}^2}{n}<\left(2+\sqrt{\sfrac{1}{\alpha}}\right)^2
    \right\}.
\end{align}
Note the implication $\mathcal{A}\subset \mathcal{A}_\smi $. We finally  introduce the interval 
\begin{align}\label{def:J_contour}
    J=\left[
    \lambda, \lambda +\norm{\partial^2\ell}_\infty\left(2+\sqrt{\sfrac{1}{\alpha}}\right)^2
    \right],
\end{align}
which encloses $\spec{H}$ on the event $\mathcal{A}$. With $J$ being chosen, we fix in the remainder of the proof a \textit{deterministic }complex contour $\Gamma\subset \{z\in \C: \Re{z}>0\}$, chosen as the boundary of the $\sfrac{\lambda}{2}-$neighborhood of $J$.  Equipped with these definitions, we can finally introduce the functional statistic
\begin{align} \label{eq:functionals_deteq}
    \mathcal{V}(z)=\Ea{\Ia \mathfrak{S}(z)},
\end{align}
which are well-defined for all $z\in\Gamma$.

\subsection{Leave-one-out approximation results}

The analysis of ERMs of the form \eqref{eq:ERM} has constituted the object of a rich body of works \citep{karoui2013asymptotic, el2013robust, donoho2016high, el2018impact, thrampoulidis2018precise, sur2019modern}. In this subsection, we remind key approximation results of \cite{karoui2013asymptotic, el2018impact} controlling the discrepancy between the ERM estimator $\hat{w}$ and its leave-one-out surrogates $\hat{w}_\smi, \tilde{w}_i$. The subsequent Appendices will build on these bounds to reach a precise asymptotic characterization of the limiting distribution of sample influences within the dataset.

\begin{proposition}[\cite{karoui2013asymptotic, el2018impact}, \cite{cui2026asymptotic}] \label{prop: LOO_approxs}
The discrepancy between the ERM estimator \eqref{eq:full_ERM} and the surrogate $\tilde{w}_i$ \eqref{eq:tildew} is controlled as 
    \begin{align}
        \underset{i\in \enm{n}}{\sup} ~\norm{\hat{w}-\tilde{w}_{i}}=O_{L_k}\left(\frac{\polylog(n)}{n}\right).
    \end{align}
We also have the following bound for the surrogate/leave-$i-$out \eqref{eq:wsmi} estimators:
\begin{align}
    \sup_{i\in\enm{n}}\norm{\tilde{w}_{i}-\hat{w}_{\smi}}= O_{L_k}\left(\frac{\polylog(n)}{\sqrt{n}}\right).
\end{align}
Furthermore, at the level of the residuals,
\begin{align}
  \underset{i\in \enm{n}}{\sup}~\underset{j\ne i}{\sup}  ~|r_{j,\smi}-r_j|=O_{L_k}\left(\frac{\polylog(n)}{\sqrt{n}}\right)
\end{align}
and
\begin{align}
     \underset{i\in \enm{n}}{\sup} |r_i-\tilde{r}_{i,i}|=O_{L_k}\left(\frac{\polylog(n)}{\sqrt{n}}\right),
\end{align}
and 
\begin{align}
    \underset{i\in \enm{n}}{\sup} \sup_{j\ne i}|\tilde{r}_{j,i}-r_{j,\smi}|=O_{L_k}\left(\frac{\polylog(n)}{\sqrt{n}}\right).
\end{align}
Finally, the cross terms can be controlled as
    \begin{align}
       \sup_{i\in\enm{n}} \sum\limits_{j\ne i} (r_j-\tilde{r}_{j,i})^2=O_{L_k}\left(\frac{\polylog(n)}{n}\right).
    \end{align}
\end{proposition}

Proposition \ref{prop: LOO_approxs} thus establishes that the surrogate $\tilde{w}_i$ \eqref{eq:tildew} provides a close approximation to the full estimator $\hat{w}$ \eqref{eq:full_ERM}, while presenting the advantage of admitting a simpler representation \eqref{eq:tildew}. We will leverage the latter in the remainder of the proof, which we report in Appendices \ref{app:concentration}, \ref{app:deterministic} and \ref{app:influence}.

\subsection{Concentration of summary statistics}
\label{app:concentration}

We have introduced above the random functional statistic $\mathfrak{S}(\cdot)$, which naturally enters in the Cauchy integral representation of key geometric descriptors of the Hessian matrix, such as its tracial moments $\mathfrak{V}^{(k)}$. We establish in this subsection the pointwise concentration of the functional $\mathfrak{S}(\cdot)$ \eqref{eq:mathfrakS}, before establishing the concentration of Cauchy integrals of this functional \eqref{eq:mathfracV}. In the subsequent Appendix \ref{app:deterministic}, we will then derive a characterization of the corresponding deterministic equivalent.

\begin{lemma}[$\mathds{1}_{\mathcal{A}}\mathfrak{S}(z)$ concentrates]\label{lem:IaSz_concentrates}
    For any $t>0$, for any $z\in \C\setminus J$ such that $\dist{z,J}>\sfrac{\lambda}{2}$, 
    \begin{align}
        \P{\left|\mathfrak{S}(z)-\Ea{ \mathfrak{S}(z)|\mathcal{A}}\right|\ge t \Big|\mathcal{A}}  \le 4e^{-\frac{n t^2}{4C^2}},\label{eq:tail_bound_Stieltjes}
    \end{align}
where $C=O(\polylog(n))$ does not depend on $z$. Its precise expression is given in equation \eqref{eq:expression_C} of the proof.
As a consequence, 
\begin{align}
   \Ia \mathfrak{S}(z)=\Ea{\Ia \mathfrak{S}(z)} +O_{L_k}\left(\frac{C}{\sqrt{n}}\right).
\end{align}
\end{lemma}

\begin{proof}
    The sub-gaussian bound is inherited from the classical concentration of Lipschitz functions of Gaussian covariates \citep{vershynin2018high, cirel2006norms}. Namely, for $\varphi:\R^{n\times d}\to \R$ a Lipschitz function with respect to the Frobenius norm, for any $t>0$,
    \begin{align}
        \P{\module{\varphi(X)-\Ea{\varphi(X)}}>t}\le 2e^{-\frac{t^2}{2\norm{\varphi}_{\rm Lip.}^2}}
    \end{align}

If $\varphi$ is only locally Lipschitz on a subdomain $\mathcal{E}\subset \R^{n\times d}$ (endowed with the Frobenius norm), Lemma 1.6 of \cite{louart2021spectral} (see also Remark 1.5 in \cite{louart2021concentration}) ensures the concentration transfers conditionally on $X\in\mathcal{E}$, namely
 \begin{align}\label{eq:louart}
        \P{\module{\varphi(X)-\Ea{\varphi(X)| X\in\mathcal{E}}}>t\Big| X\in\mathcal{E}}\le 2e^{-\frac{t^2}{2\norm{\varphi}_{\rm Lip.}^2}}
    \end{align}
This ensures one can establish concentration on the event $\mathcal{A}$, where the Stieltjes transform is well-defined and Lipschitz. We accordingly introduce $\phi_z: \mathcal{E}_{\mathcal{A}}\to \C$ with
\begin{align}
    &\phi_z(X)=\frac{1}{n}\tr[\left(\sfrac{1}{n}X^\top \Lambda(X)X +(\lambda-z)I_d\right)^{-1}],\label{eq:phiz_lip}\\
    &{\rm where}\qquad  \Lambda(X)={\rm diag}\left[\left(\partial^2\ell(\inprod{\hat{w}, x_i},\inprod{\beta, x_i})\right)_{i\in\enm{n}}\right],
\end{align}
denoting $\mathcal{E}_{\mathcal{A}}=\{M\in \R^{n\times d}\mid \norm{M}<\sqrt{n}(2+\sqrt{\sfrac{1}{\alpha}})\}$, so that $\mathcal{A}=\{X\in \mathcal{E}_{\mathcal{A}}\}$. We recall that in \eqref{eq:phiz_lip}, the estimator $\hat{w}$ is also a function of the covariate matrix $X$ through the ERM definition \eqref{eq:full_ERM}. To establish the pointwise sub-Gaussian tail bound \eqref{eq:tail_bound_Stieltjes}, it thus suffices to establish that $\phi_z$ is $O(\sfrac{1}{\sqrt{n}})$ Lipschitz, for any $z\in \C\setminus J$ with $\dist{z,J}>\sfrac{\lambda}{2}$, after what the tail bound \eqref{eq:tail_bound_Stieltjes} follows from an application of \eqref{eq:louart} to the imaginary and real parts and an union bound. To that end, we bound the derivative $\sfrac{d\phi_z}{dX}$. For any test matrix $\Delta \in\R^{n\times d}$,
\begin{align}
    \inprod{\frac{d\phi_z}{dX},\Delta}= -\frac{1}{n^2}\tr[G(z)\left(X^\top \Lambda(X)\Delta+\Delta^\top \Lambda(X)X+X^\top \inprod{\frac{d \Lambda(X)}{dX}, \Delta} X\right)G(z)]\label{eq:derivative}
\end{align}
where $\inprod{\sfrac{d}{dX}\Lambda(X), \Delta}$ is diagonal with elements
\begin{align}
    \inprod{\frac{d \Lambda(X)}{dX}, \Delta}_{kk}=\sum\limits_{i\in\enm{n}}\sum\limits_{j\in\enm{d}} \frac{d \Lambda(X)_{kk}}{dX_{ij}}\Delta_{ij}.
\end{align}
Note that the existence of this derivative is in particular guaranteed by Lemma \ref{lem:differentiability}, which ensures that the function $X\to \hat{w}$ is differentiable.
We now successively bound the three terms that make up the derivative \eqref{eq:derivative}. \paragraph{Control of $\module{\frac{1}{n^2}\tr[G(z)X^\top \Lambda(X)\Delta G(z)]}$---} The first term can be easily controlled by 
\begin{align}
    \module{\frac{1}{n^2}\tr[G(z)X^\top \Lambda(X)\Delta G(z)]}\le \frac{1}{n}\norm{\partial^2\ell}_\infty\norm{G(z)}^2\norm{X}\norm{\Delta } \le\frac{1}{\sqrt{n}}\norm{\partial^2\ell}_\infty\frac{4(2+\sqrt{\sfrac{1}{\alpha}})}{\lambda^2}\norm{\Delta}_F.
\end{align}

\paragraph{Control of $\module{\frac{1}{n^2}\tr[G(z)X^\top  \inprod{\frac{d \Lambda(X)}{dX}, \Delta} X G(z)]}$---} 
We now turn to the third term. Expounding the derivative,
\begin{align}
     \inprod{\frac{d \Lambda(X)}{dX}, \Delta}_{kk}=\sum\limits_{i\in\enm{n}}\sum\limits_{j\in[d]} \left[\partial^3\ell_k(r_k)\left(
     \delta_{ik}\hat{w}_j+\inprod{\frac{d\hat{w}}{dX_{ij}}, x_k}\right)+\partial^2\partial_2\ell_k(r_k)\delta_{ki}\beta_j\right]\Delta_{ij}.
\end{align}
The derivative of the estimator $\sfrac{d\hat{w}}{dX_{ij}}$ may be expounded starting from the stationarity condition
\begin{align}
    \lambda \hat{w}+\frac{1}{n}\sum\limits_{k\in\enm{n}} \partial\ell_k(r_k)x_k=0.
\end{align}
Taking the derivative with respect to $X_{ij}$, whose existence is guaranteed under Lemma \ref{lem:differentiability},
\begin{align}
    \lambda \frac{d\hat{w}}{dX_{ij}} + \frac{1}{n}\sum\limits_{k\in\enm{n}} \partial^2\ell_k(r_k)x_k\inprod{x_k,\frac{d\hat{w}}{dX_{ij}}}+ \frac{1}{n}\partial\ell_i(r_i)e_j+\frac{1}{n}x_i(\partial^2\ell_i(r_i)\hat{w}_{j}+\partial\partial_2\ell_i(r_i)\beta_j)=0.
\end{align}
We denoted $e_j\in \R^d$ the $j-$th canonical basis vector. Thus,
\begin{align}
    \inprod{\frac{d\hat{w}}{dX_{ij}},x_k}=-\partial\ell_i(r_i)\frac{\inprod{x_k, H^{-1}e_j}}{n}-\frac{\inprod{x_k, H^{-1}x_i}}{n}(\partial^2\ell_i(r_i)\hat{w}_{j}+\partial\partial_2\ell_i(r_i)\beta_j).
\end{align}
Introducing for compactness of notation the vector $h\in\R^n$ and matrices $E \in\R^{n\times d}, F\in\R^{n\times n}, J\in\R^{n\times n}$ with entries
\begin{align}
    &h_i=\frac{1}{\sqrt{n}}\partial\ell_i(r_i), && E_{kj}=-\partial^3\ell_k(r_k)\frac{\inprod{x_k, H^{-1}e_j}}{\sqrt{n}},\\
    &F_{ki}=-\partial^2\ell_i(r_i)\partial^3\ell_k(r_k)\frac{\inprod{x_k, H^{-1}x_i}}{n}, && L_{ki}=-\partial\partial_2\ell_i(r_i)\partial^3\ell_k(r_k)\frac{\inprod{x_k, H^{-1}x_i}}{n},
\end{align}
one can rewrite the diagonal matrix $\inprod{\sfrac{d \Lambda(X)}{dX}, \Delta}$, viewed as a $\R^n$ vector, as
\begin{align}
    \vect{\inprod{\frac{d \Lambda(X)}{dX}, \Delta}}\!=\! E\Delta^\top h+F\Delta \hat{w}+J\Delta \beta+ \mathrm{diag}((\partial^3\ell_k(r_k))_k)\Delta \hat{w}+\mathrm{diag}((\partial^2\partial_2\ell_k(r_k))_k)\Delta \beta.
\end{align}
Thus
\begin{align}
    \norm{ \vect{\inprod{\frac{d \Lambda(X)}{dX}, \Delta}}}\le \left[\norm{\partial \ell}_\infty\norm{E}+\norm{F}\norm{\hat{w}}+\norm{L}+\norm{\partial^3\ell}_\infty\norm{\hat{w}}+\norm{\partial^2\partial_2\ell}_\infty\right] \norm{\Delta}_F .\label{eq:bound_vec_DLambdX}
\end{align}
$\norm{\hat{w}}$ is deterministically upper-bounded by a constant $M$ (Assumption \ref{ass:loss}) through Lemma \ref{lemma:normw} for any $X$. On the other hand,
\begin{align}
    &\norm{E}\le \norm{\partial^3\ell}_\infty \frac{2+\sqrt{\sfrac{1}{\alpha}}}{\lambda },\\
    &\norm{F}\le \norm{\partial^2\ell}_\infty\norm{\partial^3\ell}_\infty \frac{(2+\sqrt{\sfrac{1}{\alpha}})^2}{\lambda }\\
    &\norm{L}\le \norm{\partial\partial_2\ell}_\infty\norm{\partial^3\ell}_\infty \frac{(2+\sqrt{\sfrac{1}{\alpha}})^2}{\lambda }
\end{align}
since we remind $X\in \mathcal{E}_{\mathcal{A}}$ and thus has $(2+\sqrt{\sfrac{1}{\alpha}})-$bounded operator norm. Taken together, these bounds and \eqref{eq:bound_vec_DLambdX} yield $O(\polylog(n))\norm{\Delta}_F$ control over $\norm{ \vect{\inprod{\sfrac{d \Lambda(X)}{dX}, \Delta}}}$. Coming back to the initial objective, one is now in a position to build on this control to bound $\module{\sfrac{1}{n^2}\tr[G(z)X^\top  \inprod{\sfrac{d \Lambda(X)}{dX}, \Delta} X G(z)]}$ as 
\begin{align}
    \module{\frac{1}{n^2}\tr[G(z)X^\top  \inprod{\frac{d \Lambda(X)}{dX}, \Delta} X G(z)]}\le \frac{1}{n}\module{\inprod{\vect{\inprod{\sfrac{d \Lambda(X)}{dX}, \Delta}}, \mathrm{vec ~diag}\left[
    \frac{XG(z)^2 X^\top}{n}
    \right]}},
\end{align}
where vect diag refers to the diagonal viewed as a $\R^n$ vector. The desired bound then follows from a Cauchy-Schwarz inequality, remarking that 
\begin{align}
    \norm{\mathrm{vect ~diag}\left[
    \frac{XG(z)^2 X^\top}{n}
    \right]}\le \norm{G(z)}^2 \sqrt{\sum\limits_{i\in\enm{n}}\frac{\norm{x_i}^4}{n^2}}\le \frac{4}{\lambda^2}\sqrt{n}(2+\sqrt{\sfrac{1}{\alpha}})^2.
\end{align}
\paragraph{Lipschitzness of $\phi_z$ --- }Putting everything together, for any $z\in\C\setminus J$ satisfying the requirements of the lemma, and any test matrix $\Delta$,
\begin{align}
     \module{\inprod{\frac{d\phi_z}{dX},\Delta}}\le \frac{C}{\sqrt{n}}\norm{\Delta}_F,
\end{align}
where $C=O(\polylog(n))$ is independent of $z,\Delta$ and explicitly given by
\begin{align}\label{eq:expression_C}
    C=\frac{4t_\alpha}{\lambda^2}\left[    2\norm{\partial^2\ell}_\infty+t_\alpha\left(\norm{\partial \ell}_\infty\norm{\partial^3\ell}_\infty\frac{t_\alpha}{\lambda}+\norm{\partial^3\ell}_\infty\frac{t_\alpha^2}{\lambda}\left(M+1\right)+\norm{\partial^3\ell}_\infty M+\norm{\partial^2\partial_2\ell}_\infty\right)
    \right]
\end{align}
where $t_\alpha=(2+\sqrt{\sfrac{1}{\alpha}})$. One can thus conclude that $\phi_z$ is $\sfrac{C}{\sqrt{n}}-$ Lipschitz, from which the sub-Gaussian tail \eqref{eq:tail_bound_Stieltjes} follows. 
\paragraph{Proof of the second claim ---} Finally, we transfer the $\mathcal{A}-$ conditional tail bound \eqref{eq:tail_bound_Stieltjes} onto the truncated random variable $\Ia\mathfrak{S}(z)$. The unconditional moments of the residual $\Ia\mathfrak{S}(z)-\Ea{\Ia\mathfrak{S}(z)|\mathcal{A}}$ are related to the conditional moments as
\begin{align}
    &\module{\Ea{\module{\Ia\mathfrak{S}(z)-\Ea{\Ia\mathfrak{S}(z)|\mathcal{A}}}^k}- \Ea{\module{\mathfrak{S}(z)-\Ea{\mathfrak{S}(z)|\mathcal{A}}}^k\Big|\mathcal{A}}}\\
    &\le\qquad 
\module{\Ea{\module{\Ia\mathfrak{S}(z)-\Ea{\Ia\mathfrak{S}(z)|\mathcal{A}}}^k}- \Ea{\module{\mathfrak{S}(z)-\Ea{\mathfrak{S}(z)|\mathcal{A}}}^k\Ia}}\\
&\qquad +\module{ \Ea{\module{\mathfrak{S}(z)-\Ea{\mathfrak{S}(z)|\mathcal{A}}}^k\Big|\mathcal{A}}- \Ea{\module{\mathfrak{S}(z)-\Ea{\mathfrak{S}(z)|\mathcal{A}}}^k\Ia}}\le  6\left(\frac{4}{\lambda}\right)^k e^{-\frac{n}{2}},
\end{align}
for any $n\ge 4$. Using the tail bound \eqref{eq:tail_bound_Stieltjes} to evaluate the conditional moment, one finally reaches 
\begin{align}
    \Ea{\module{\Ia\mathfrak{S}(z)-\Ea{\Ia\mathfrak{S}(z)|\mathcal{A}}}^k}=O(C^kn^{-\sfrac{k}{2}}).
\end{align}
Thus, 
\begin{align}
    \Ia\mathfrak{S}(z)=\Ea{\Ia\mathfrak{S}(z)|\mathcal{A}}+O_{L_k}\left(\frac{C}{\sqrt{n}}\right)=\Ea{\Ia\mathfrak{S}(z)}+O(e^{-\frac{n}{2}})+O_{L_k}\left(\frac{C}{\sqrt{n}}\right),
\end{align}
concluding the proof.
\end{proof}

The pointwise concentration properties established for the functionals $\Ia\mathfrak{S}(\cdot)$, proven in Lemmas \ref{lem:IaSz_concentrates}, will prove instrumental when computing the corresponding deterministic equivalents $ \mathcal{V}(\cdot)$ in Appendix \ref{app:deterministic}. As such, they however do not suffice to deduce the concentration of the integral $\mathfrak{V}^{(k)}$ \eqref{eq:mathfracV}. Instead, we use the Efron-Stein lemma \citep{efron1981jackknife} to directly establish the concentration of the latter. This is the object of Lemma \ref{lem:Vz_concentrates}.

\begin{lemma}[$\mathfrak{V}^{(p)}$  concentrates]\label{lem:Vz_concentrates} 
    Let $\mathcal{C} \subset \{z\in \C: \Re{z}>0\}$ be the boundary of the $\sfrac{\lambda}{2}-$ neighborhood of the (random) interval
    \begin{align}
        \mathfrak{J}=\left[\lambda, \norm{\partial^2\ell}_\infty \frac{\norm{X}^2}{n}+\lambda\right]. \label{eq:contour}
    \end{align}
Then for any $p\in\mathbb{N}$,
\begin{align}
    \Var{\frac{1}{2\pi i}\oint_\mathcal{C} \frac{1}{z^p}\mathfrak{S}(z) dz }=O\left(\frac{\polylog(n)}{n}\right).
\end{align}
The variance also bears over both random variables $\mathcal{C}, \mathfrak{S}(z)$.
\end{lemma}
\begin{proof} The proof will leverage the results of Proposition \ref{prop: LOO_approxs} in order to approximate the full Stieltjes integral by easier-to-handle integrals bearing over the surrogate and leave-one-out approximations. First,  note that $\spec{H},\spec{\tilde{H}},\spec{H_\smi}\subset \mathfrak{J}$, and thus all the resolvent in the proof are well-defined.  \\

From the Efron-Stein Lemma, 
\begin{align}
      &\Var{\frac{1}{2\pi i}\oint_{\mathcal{C} } \frac{1}{z^p}\mathfrak{S}(z) dz } \\
      &\le \sum\limits_{i\in\enm{n}} \Ea{\module{\frac{1}{2\pi i}\oint_{\mathcal{C} } \frac{1}{z^p}(\mathfrak{S}(z)-\mathfrak{S}_\smi(z))dz }^2 }\\
      &\le \sum\limits_{i\in\enm{n}} 2\Ea{\module{\frac{1}{2\pi i}\oint_{\mathcal{C} } \frac{1}{z^p}(\mathfrak{S}(z)-\tilde{\mathfrak{S}}_i(z))dz }^2 }+2\Ea{\module{\frac{1}{2\pi i}\oint_{\mathcal{C} } \frac{1}{z^p}(\mathfrak{S}_\smi(z)-\tilde{\mathfrak{S}}_i(z))dz }^2 } .\label{eq:Vz_EF}
\end{align}
In the first line, we used the fact that the leave$-i-$out contour $\mathcal{C}_\smi$, defined as the boundary of the $\sfrac{\lambda}{2}-$neighborhood of $\left[\lambda, \norm{\partial^2\ell}_\infty \sfrac{\norm{X_\smi}^2}{n}+\lambda\right]$, could be deformed into $\mathcal{C}$, since $X \succeq X_\smi$. 
Let us remind that the expectations also bear over the random contour $\mathcal{C} $.
We successively control the two terms. 
\paragraph{First term of \eqref{eq:Vz_EF} ---} First,
\begin{align}
    \module{\frac{1}{2\pi i}\oint_{\mathcal{C} } \frac{1}{z^p}(\mathfrak{S}(z)-\tilde{\mathfrak{S}}_i(z))dz } \le \frac{\length{\mathcal{C} }}{2\pi }\sup_{z\in \mathcal{C} } \module{\mathfrak{S}(z)-\tilde{\mathfrak{S}}_i(z)}\sup_{z\in \mathcal{C} } \frac{1}{\module{z}^p}\le \frac{\length{\mathcal{C} }2^p}{2\pi\lambda^p }\sup_{z\in \mathcal{C} } \module{\mathfrak{S}(z)-\tilde{\mathfrak{S}}_i(z)},
\end{align}
denoting $\length{\mathcal{C} }$ the length of the contour, and using the estimation lemma. Lemmas \ref{lem:cov} and \ref{lemma:normx} provide control of the latter
\begin{align}
    \length{\mathcal{C} }=O\left(\polylog(n)\right).
\end{align}
We now need to control the maximal Stieltjes approximation error $\module{\mathfrak{S}(z)-\tilde{\mathfrak{S}}_i(z))}$ along the integral. But 
\begin{align}
    &\sup_{z\in \mathcal{C} }\module{\mathfrak{S}(z)-\tilde{\mathfrak{S}}_i(z)}\\
    &=\sup_{z\in \mathcal{C} }\module{\frac{1}{n}\tr[G(z)\left(H-\tilde{H}_i\right)\tilde{G}_i(z)]}\\
    &=\sup_{z\in \mathcal{C} }\module{\frac{1}{n}\sum\limits_{j\in\enm{n}\smi}\partial^3\ell_j(\check{r}_j)(r_j-\tilde{r}_{j,i})\frac{\inprod{G(z)x_j, \tilde{G}_i(z)x_j}}{n} +\frac{1}{n}\partial^2\ell_i(r_i)\frac{\inprod{G(z)x_i, \tilde{G}_i(z)x_i}}{n}}\\
    &\le\left( \frac{1}{\sqrt{n}}\sqrt{\sum\limits_{j\in\enm{n}\smi}(r_j-\tilde{r}_{j,i})^2} \norm{\partial^3\ell}_\infty +\frac{1}{n}\norm{\partial^2\ell}_\infty\right)\sup_{j\in\enm{n}}\frac{\norm{x_j}^2}{n} \sup_{z\in \mathcal{C} } \norm{G(z)} \sup_{z\in \mathcal{C} } \norm{\tilde{G}_i(z)},
\end{align}
where $\check{r}_j$ in the unordered interval $ (\tilde{r}_{j,i},r_{j}) $.
We used the Cauchy-Schwartz inequality in the last line.
The first term is control by Proposition \ref{prop: LOO_approxs}. The choice of $\mathcal{C} $ furthermore ensures $\sup_{z\in \mathcal{C} } \norm{G(z)} \le \sfrac{2}{\lambda}$ and $\sup_{z\in \mathcal{C} } \norm{\tilde{G}_i(z)} \le \sfrac{2}{\lambda}$. Assembling these bounds yields the control
\begin{align}\label{eq:Vz_conc_ES1}
     \module{\frac{1}{2\pi i}\oint_{\mathcal{C} } \frac{1}{z^p}(\mathfrak{S}(z)-\tilde{\mathfrak{S}}_i(z))dz } =O_{L_k}\left(\frac{\polylog(n)}{n}\right).
\end{align}
\paragraph{Second term of \eqref{eq:Vz_EF} ---} We now address the leave-one-out/surrogate discrepancy term of \eqref{eq:Vz_EF}. Similarly to the first term, the contour integral can be controlled by the uniform supremum over $z\in\mathcal{C} $ of the difference $|\tilde{\mathfrak{S}}_i(z)-\mathfrak{S}_\smi (z)|$:
 \begin{align}
    \module{\frac{1}{2\pi i}\oint_{\mathcal{C} } \frac{1}{z^p}(\tilde{\mathfrak{S}}_i(z)-\mathfrak{S}_\smi (z))dz } \le \frac{\length{\mathcal{C} }2^p}{2\pi\lambda^p }\sup_{z\in \mathcal{C} } \module{\tilde{\mathfrak{S}}_i(z)-\mathfrak{S}_\smi (z)}.
\end{align}
But the latter reads
\begin{align}
    \sup_{z\in \mathcal{C} }\left|\tilde{\mathfrak{S}}_i(z)-\mathfrak{S}_\smi (z)\right|\le & \sup_{z\in \mathcal{C} }\left|\frac{1}{n}\sum\limits_{j\in\enm{n}\setminus i}  \partial^3\ell_j(r_{j, \smi})(\tilde{r}_{j,i}-r_{j, \smi})\frac{\inprod{\tilde{G}_i(z)x_j, G_\smi(z)x_j}}{n}\right|\\
&+ \sup_{z\in \mathcal{C} }\left|\frac{1}{n}\sum\limits_{j\in\enm{n}\setminus i}  \partial^4\ell_j(\check{r}_{j})(\tilde{r}_{j,i}-r_{j, \smi})^2\frac{\inprod{\tilde{G}_i(z)x_j, G_\smi(z)x_j}}{n}\right| \label{eq:Vconc_deconmpo_ES2}
\end{align}
for some $\check{r}_j$ in the unordered interval $ (\tilde{r}_{j,i},r_{j, \smi}) $. %The first term of \eqref{eq:Vconc_deconmpo_ES2} is $O_{L_k}(\sfrac{\polylog(n)}{n})$, following from the control of the operator norms $\sup_{z\in \mathcal{C} }\norm{\tilde{G}_i(z)},\sup_{z\in \mathcal{C} }\norm{G_\smi(z)} \le \sfrac{2}{\lambda}$ ensured by the choice of the random contour $\mathcal{C} $. 
The last term of \eqref{eq:Vconc_deconmpo_ES2}  scales as $O_{L_k}(\sfrac{\polylog(n)}{n})$, using Proposition \ref{prop: LOO_approxs}. The second term of \eqref{eq:Vconc_deconmpo_ES2} warrants more careful control of the correlations within the sum. A first simplification follows from the identity
\begin{align}
    \sup_{z\in \mathcal{C} } \norm{\tilde{G}_i(z)-G_\smi(z)}=O_{L_k}\left(\frac{\polylog(n)}{\sqrt{n}}\right),
\end{align}
which guarantees 
\begin{align}
    &\sup_{z\in \mathcal{C} } \left|\frac{1}{n}\sum\limits_{j\in\enm{n}\setminus i}  \partial^3\ell_j(r_{j, \smi})(\tilde{r}_{j,i}-r_{j, \smi})\frac{\inprod{\tilde{G}_i(z)x_j, G_\smi(z)x_j}}{n}\right|\\
    &\le \sup_{z\in \mathcal{C} } \left|\frac{1}{n}\sum\limits_{j\in\enm{n}\setminus i}  \partial^3\ell_j(r_{j, \smi})(\tilde{r}_{j,i}-r_{j, \smi})\frac{\inprod{G_\smi (z)x_j, G_\smi(z)x_j}}{n}\right|+O_{L_k}\left(\frac{\polylog(n)}{n}\right).
\end{align}
We can thus limit ourselves to the study of the latter term. Unwrapping the expression of the surrogate residual $\tilde{r}_{j,i}$ yields the explicit expression
\begin{align}
    &\left|\frac{1}{n}\sum\limits_{j\in\enm{n}\setminus i}  \partial^3\ell_j(r_{j, \smi})(\tilde{r}_{j,i}-r_{j, \smi})\frac{\inprod{G_\smi (z)x_j, G_\smi(z)x_j}}{n}\right|\\
    &=\left|\frac{1}{n^2}\sum\limits_{j\in\enm{n}\setminus i}  \partial^3\ell_j(r_{j, \smi})\partial \ell_i(\tilde{r}_{i,i}) \inprod{x_j, H_\smi^{-1}x_i}\frac{\inprod{G_\smi (z)x_j, G_\smi(z)x_j}}{n}\right|.
\end{align}
Note that it then follows that
\begin{align}
    &\Ea{\left(\frac{1}{n}\sum\limits_{j\in\enm{n}\setminus i}  \partial^3\ell_j(r_{j, \smi})(\tilde{r}_{j,i}-r_{j, \smi})\frac{\inprod{G_\smi (z)x_j, G_\smi(z)x_j}}{n}\right)^2}\\
    &\le \norm{\partial\ell}_\infty^2\Ea{\norm{\frac{1}{n^2}\sum\limits_{j\in\enm{n}\setminus i}  \partial^3\ell_j(r_{j, \smi})\frac{\inprod{G_\smi (z)x_j, G_\smi(z)x_j}}{n}H_\smi^{-1}x_j}^2}=O\left(\frac{\polylog(n)}{n^2}\right).
\end{align}
One needs to transfer this second-moment bound to the difference of integrals, facing the hurdle that the latter bears over a $X-$dependent contour $\mathcal{C}$. To circumvent the issue, we note that
\begin{align}
    &\module{\frac{1}{2\pi i}\oint_{\mathcal{C} } \frac{1}{z^p}(\tilde{\mathfrak{S}}_i(z)-\mathfrak{S}_\smi (z))dz }\\
    &=
     \module{\frac{1}{2\pi i}\oint_{\mathcal{C} } \frac{1}{z^p} \frac{1}{n^2}\sum\limits_{j\in\enm{n}\setminus i}  \partial^3\ell_j(r_{j, \smi})\partial \ell_i(\tilde{r}_{i,i}) \inprod{x_j, H_\smi^{-1}x_i}\frac{\inprod{G_\smi (z)x_j, G_\smi(z)x_j}}{n} dz }+O_{L_k}\left(
     \frac{\polylog(n)}{n}
     \right)\\
     &=
     \module{\frac{1}{2\pi i}\oint_{\Gamma } \frac{1}{z^p} \frac{1}{n^2}\sum\limits_{j\in\enm{n}\setminus i}  \partial^3\ell_j(r_{j, \smi})\partial \ell_i(\tilde{r}_{i,i}) \inprod{x_j, H_\smi^{-1}x_i}\Iai \frac{\inprod{G_\smi (z)x_j, G_\smi(z)x_j}}{n} dz }+O_{L_k}\left(
     \frac{\polylog(n)}{n}
     \right),
\end{align}
thereby effectively swapping the random contour $\mathcal{C}$ for the deterministic contour $\Gamma$, at the cost of the introduction of the indicator $\Iai$. One can then write
\begin{align}
    &\Ea{\module{\frac{1}{2\pi i}\oint_{\mathcal{C} } \frac{1}{z^p}(\tilde{\mathfrak{S}}_i(z)-\mathfrak{S}_\smi (z))dz }^2}\\
    &\le \Ea{\frac{\norm{\partial\ell}_\infty^2}{2\pi^2}\norm{\oint_{\Gamma } \frac{1}{z^p} \frac{1}{n^2}\sum\limits_{j\in\enm{n}\setminus i}  \partial^3\ell_j(r_{j, \smi})\Iai \frac{\inprod{G_\smi (z)x_j, G_\smi(z)x_j}}{n} H_\smi^{-1}x_j dz }^2}+O\left(
     \frac{\polylog(n)}{n^2}
     \right).
\end{align}
But 
\begin{align}
    &\Ea{\norm{\oint_{\Gamma } \frac{1}{z^p} \frac{1}{n^2}\sum\limits_{j\in\enm{n}\setminus i}  \partial^3\ell_j(r_{j, \smi})\Iai \frac{\inprod{G_\smi (z)x_j, G_\smi(z)x_j}}{n} H_\smi^{-1}x_j dz }^2}+O\left(
     \frac{\polylog(n)}{n^2}
     \right)\\
     &\le \Ea{\left(\frac{\mathcal{P}_\Gamma 2^p}{\lambda^p}
     \sup_{z\in \Gamma} \norm{ \frac{1}{n^2}\sum\limits_{j\in\enm{n}\setminus i}  \partial^3\ell_j(r_{j, \smi})\Iai \frac{\inprod{G_\smi (z)x_j, G_\smi(z)x_j}}{n} H_\smi^{-1}x_j}
     \right)^2}\\
     &=O\left(\frac{\polylog(n)}{n^2}\right).
\end{align}
Assembling these successive bounds, finally yields the control
\begin{align}\label{eq:Vz_conc_ES2}
    \Ea{\module{\frac{1}{2\pi i}\oint_{\mathcal{C} } \frac{1}{z^p}(\tilde{\mathfrak{S}}_i(z)-\mathfrak{S}_\smi (z))dz }^2} =O\left(\frac{\polylog(n)}{n^2}\right).
\end{align}
Injecting the bounds \eqref{eq:Vz_conc_ES1} and \eqref{eq:Vz_conc_ES2} into the Efron-Stein lemma \eqref{eq:Vz_EF} concludes the proof.
\end{proof}

\subsection{Deterministic equivalents}
\label{app:deterministic}
We have established in Appendix \ref{app:concentration} the concentration of the random statistics $\mathfrak{V}^{(k)}$ \eqref{eq:mathfracV}. It is hence sufficient to ascertain their expectations. We note that the characterization of the particular parameter $V^{(1)}$ constitutes a now classical result featured in a large body of works, e.g. \citep{karoui2013asymptotic, thrampoulidis2018precise, donoho2016high}. \\

In the following, we remind $\Gamma\subset \{z\in \C: \Re{z}>0\}$ is a \textit{deterministic} contour enclosing the interval 
\begin{align}
    J=\left[
    \lambda, \lambda +\norm{\partial^2\ell}_\infty\left(2+\sqrt{\sfrac{1}{\alpha}}\right)^2
    \right]
\end{align}
at a distance $\dist{\Gamma, J}\ge \sfrac{\lambda}{2}$, while also satisfying $\dist{\Gamma, 0} \ge \sfrac{\lambda}{2}$.

\begin{lemma}[Characterization of $V^{(k)}$]\label{lem:Charac_Vk}
    The summary statistics $V^{(k)}$ are characterized by 
\begin{align}
    V^{(k)}=\frac{-1}{2\pi i}\oint_\Gamma \frac{1}{z^k}\mathcal{V}(z) dz+O\left(\polylog(n)e^{-\sfrac{n}{2}}\right),
\end{align}
where we recall $\mathcal{V}(z)=\Ea{\Ia\mathfrak{S}(z)}$ \eqref{eq:functionals_deteq}.

\end{lemma}

\begin{proof}
We showed in Lemma \ref{lem:Vz_concentrates} the concentration of the random statistic \begin{align}
    \mathfrak{V}^{(k)}=\frac{-1}{2\pi i}\oint_{\mathcal{C} } \frac{1}{z^k}\mathfrak{S}(z) dz,
\end{align}
which involves an integration over the \textit{random} contour $\mathcal{C}$, defined in \eqref{eq:contour}. With a view to ascertaining its expectation, a first step consists in replacing the random contour $\mathcal{C}$ by a \textit{deterministic} contour $\Gamma$. The main idea is that $\mathcal{C}$ needs to enclose the spectrum of the Hessian $H$. In the proportional asymptotic limit, the support of $H$ is upper bounded with high probability, allowing one to interchange $\mathcal{C}$ and $\Gamma$. 
Note that the event $\mathcal{A}$ implies $\spec{H}\subset J $. Partitioning the expectation,
\begin{align}
    \Ea{\mathfrak{V}^{(k)}}=\Ea{\mathfrak{V}^{(k)} \mathds{1}_{\mathcal{A}}}+\Ea{\mathfrak{V}^{(k)} \mathds{1}_{\mathcal{A}^c}}
\end{align}
The second term can be bounded as
\begin{align}
    \module{\Ea{\mathfrak{V}^{(k)} \mathds{1}_{\mathcal{A}^c}}}\le \frac{1}{\alpha \lambda ^k}\P{\mathcal{A}^c} \le O\left(\polylog(n)e^{-\sfrac{n}{2}}\right).
\end{align}
 From the Fubini-Tonelli theorem, the expectation on the event $\mathcal{A}$ can be expressed as
\begin{align}
    \Ea{\mathfrak{V}^{(k)} \mathds{1}_{\mathcal{A}}}=\frac{-1}{2\pi i}\oint_\Gamma \frac{1}{z^k}\Ea{\mathds{1}_{\mathcal{A}}\mathfrak{S}(z)} dz=\frac{-1}{2\pi i}\oint_\Gamma \frac{1}{z^k}\mathcal{V}(z)dz
\end{align}
\end{proof}

Lemma \ref{lem:Charac_Vk}  thus express the (deterministic) summary statistics $V^{(k)}$ in terms of the functional $\mathcal{V}(\cdot)$ \eqref{eq:functionals_deteq}. The following builds characterizations for the latter in terms of $V^{(k)}$ and $Q$, with the characterization of the latter \eqref{eq:charac_Q} being known from e.g. \cite{karoui2013asymptotic, el2018impact, thrampoulidis2018precise, cui2026asymptotic}, thereby closing the equations. We stress that the characterization \eqref{eq:charac_Vz} of $\mathcal{V}(\cdot)$ (and the convergence of $\mathfrak{S}(\cdot)$ thereto) is already known in a stronger form from the work of \cite{asgari2025local} (equations (22, 23)). For self-containedness, we provide below a derivation of the weaker form \eqref{eq:charac_Vz}, in the notations of the present paper.

\begin{lemma}[Characterization of $\mathcal{V}(\cdot)$]\label{lem:characVz}
   For any $z\in\Gamma$ with $\Im{z}\ne 0$, the functional  $\mathcal{V}(\cdot)$ satisfies asymptotically the equation
\begin{align}
   \frac{1}{\alpha}-(\lambda-z)\mathcal{V}(z)-\Ea{\frac{\partial^2\ell\left(r,g_2\right)\mathcal{V}(z)}{1+\partial^2\ell\left(r,g_2\right)\mathcal{V}(z)}}=O\left(\frac{\polylog(n)}{n^{\frac{1}{4}}}\right).
\end{align}
In the above display, the expectations bear over the joint Gaussian variables
\begin{align}
    g_1, g_2\sim\mathcal{N}\left(0_2, Q\right),
\end{align}
where the covariance $Q$ satisfies \eqref{eq:charac_Q}. We employed the shorthand $r=\prox_{V^{(1)}\ell(\cdot, g_2)}(g_1)$. 
\end{lemma}

\begin{proof}
    For any $z\in\Gamma$, we start from the identity
    \begin{align}
\Ia I_d=(\lambda-z)\Ia G(z)+\frac{1}{n}\sum\limits_{i\in\enm{n}}\partial^2\ell_i(r_i)x_ix_i^\top \Ia G(z)
    \end{align}
One has, by taking the trace and using the Woodbury inversion lemma,
\begin{align}
   \Ia\left( \frac{1}{\alpha}-(\lambda-z)\mathfrak{S}(z)\right)&=\frac{1}{n}\sum\limits_{i\in\enm{n}}\partial^2\ell_i(r_i) \frac{\inprod{x_i, \Ia G(z)x_i}}{n}\\
    &=\frac{1}{n}\sum\limits_{i\in\enm{n}}\frac{\partial^2\ell_i(r_i)\frac{ \inprod{x_i, \Ia\hat{G}_i(z)x_i}}{n}}{1+\partial^2\ell_i(r_i)\frac{\inprod{x_i, \Ia  \hat{G}_i(z)x_i}}{n}},\label{eq:after_Wood}
\end{align}
Where $\hat{G}_i(z)=(\hat{H}_i-zI_d)^{-1}$ is the resolvent of the Hessian variant
\begin{align}
    \hat{H}_i=\frac{1}{n}\sum\limits_{j\ne i}\partial^2\ell_j(r_j)x_jx_j^\top+\lambda I_d=H-\frac{1}{n}\partial^2\ell_i(r_i)x_ix_i^\top.\label{eq:hatHi}
\end{align}
We note that $\hat{G}_i(z)$ is indeed well-defined for all $z\in\Gamma$ under the event $\mathcal{A}$, since $H\succeq\hat{H}_i$, and thus $\spec{\hat{H}_i}\subset J$ under the event $\mathcal{A}$. Notably, $\dist{z,\spec{\hat{H}_i}}\ge \sfrac{\lambda}{2}$.

\paragraph{Approximating $r_i$ --- }The following step is to approximate the residual $r_i$ by the surrogate $\tilde{r}_{i,i}$ in \eqref{eq:after_Wood}. Bounding the corresponding correction term yields
\begin{align}
    &\sup_{i\in\enm{n}}\module{\frac{\partial^2\ell_i(r_i)\frac{ \inprod{x_i, \Ia\hat{G}_i(z)x_i}}{n}}{1+\partial^2\ell_i(r_i)\frac{\inprod{x_i, \Ia  \hat{G}_i(z)x_i}}{n}}-\frac{\partial^2\ell_i(\tilde{r}_{i,i})\frac{ \inprod{x_i, \Ia\hat{G}_i(z)x_i}}{n}}{1+\partial^2\ell_i(\tilde{r}_{i,i})\frac{\inprod{x_i, \Ia  \hat{G}_i(z)x_i}}{n}}} \\
    &\le \sup_{i\in\enm{n}}\frac{1}{\module{1+\partial^2\ell_i(r_i)\frac{\inprod{x_i, \Ia  \hat{G}_i(z)x_i}}{n}}}\sup_{i\in\enm{n}}\frac{1}{\module{1+\partial^2\ell_i(\tilde{r}_{i,i})\frac{\inprod{x_i, \Ia  \hat{G}_i(z)x_i}}{n}}} \\
    &\qquad \times \norm{\partial^3\ell}_\infty\sup_{i\in \enm{n}}\module{r_i-\tilde{r}_{i,i}}\sup_{i\in \enm{n}} \module{\frac{\inprod{x_i, \Ia  \hat{G}_i(z)x_i}}{n}}.
\end{align}
Note that the term $\sup_{i\in \enm{n}}\module{r_i-\tilde{r}_{i,i}}$ is controlled by Proposition \ref{prop: LOO_approxs}, while for the last term
\begin{align}
    \module{\frac{\inprod{x_i, \Ia  \hat{G}_i(z)x_i}}{n}} \le \frac{2}{\lambda } (2+\sqrt{\sfrac{1}{\alpha}})^2.
\end{align}

 We now exhibit a lower-bound on $1+\partial^2\ell_i(r_i)\sfrac{\inprod{x_i, \Ia  \hat{G}_i(z)x_i}}{n}$. From the Woodbury formula,
\begin{align}
    \Ia G(z)x_i=\frac{1}{1+\partial^2\ell_i(r_i)\Ia\frac{\inprod{x_i,  \hat{G}_i(z)x_i}}{n}} \Ia\hat{G}_i(z)x_i.
\end{align}
Thus, if $\norm{x_i}\ne 0$, one can lower-bound
\begin{align}
    \module{1+\partial^2\ell_i(r_i)\Ia\frac{\inprod{x_i,  \hat{G}_i(z)x_i}}{n}}\ge \Ia \frac{\norm{\hat{G}_i(z)x_i}}{\norm{G(z)x_i}} +(1-\Ia)\ge \Ia \frac{\norm{\hat{G}_ix_i}}{\sfrac{2}{\lambda}\norm{x_i}}+(1-\Ia).
\end{align}
 Since furtheremore when $\Ia=1$
\begin{align}
    \norm{x_i}\le \norm{\hat{H}_i-zI_d} \norm{\hat{G}_i x_i}\le \left(\lambda +\norm{\partial^2\ell}_\infty (2+\sqrt{\sfrac{1}{\alpha}})^2+\module{z}\right)\norm{\hat{G}_i x_i},
\end{align}
one finally has
\begin{align}
    \module{1+\partial^2\ell_i(r_i)\Ia\frac{\inprod{x_i,  \hat{G}_i(z)x_i}}{n}}\ge  \frac{1}{\sfrac{2}{\lambda}\left(\lambda +\norm{\partial^2\ell}_\infty (2+\sqrt{\sfrac{1}{\alpha}})^2+\module{z}\right)}\wedge 1  \label{eq:denom_xGx}
\end{align}
This inequality holds trivially true when $\norm{x_i}=0$. 
The same reasoning holds to reach a lower bound on $1+\partial^2\ell_i(\tilde{r}_{i,i})\sfrac{\inprod{x_i, \Ia  \hat{G}_i(z)x_i}}{n}$, using this time the identity
\begin{align}
    \Ia \check{G}(z)x_i=\frac{1}{1+\partial^2\ell_i(\tilde{r}_{i,i})\Ia\frac{\inprod{x_i,  \hat{G}_i(z)x_i}}{n}} \Ia\hat{G}_i(z)x_i.
\end{align}
introducing for that purpose the modified resolvent
\begin{align}
    \check{G}(z)=\left[\frac{1}{n}\sum\limits_{j\ne i}\partial^2\ell(r_j)x_jx_j^\top +\frac{1}{n}\partial^2\ell(\tilde{r}_{i,i})x_ix_i^\top+(\lambda -z)I_d \right]^{-1},
\end{align}
which also satisfies the operator norm bound $\norm{\check{G}(z)}\le \sfrac{2}{\lambda}$ on the event $\mathcal{A}$. Repeating the above argument,
\begin{align}
    \module{1+\partial^2\ell_i(\tilde{r}_{i,i})\Ia\frac{\inprod{x_i,  \hat{G}_i(z)x_i}}{n}}\ge  \frac{1}{\sfrac{2}{\lambda}\left(\lambda +\norm{\partial^2\ell}_\infty (2+\sqrt{\sfrac{1}{\alpha}})^2+\module{z}\right)}\wedge 1  \label{eq:bound_denom_quad}
\end{align}
and thus
\begin{align}
    \sup_{i\in\enm{n}}\module{\frac{\partial^2\ell_i(r_i)\frac{ \inprod{x_i, \Ia\hat{G}_i(z)x_i}}{n}}{1+\partial^2\ell_i(r_i)\frac{\inprod{x_i, \Ia  \hat{G}_i(z)x_i}}{n}}-\frac{\partial^2\ell_i(\tilde{r}_{i,i})\frac{ \inprod{x_i, \Ia\hat{G}_i(z)x_i}}{n}}{1+\partial^2\ell_i(\tilde{r}_{i,i})\frac{\inprod{x_i, \Ia  \hat{G}_i(z)x_i}}{n}}} =O_{L_k}\left(\frac{\polylog(n)}{\sqrt{n}}\right)
\end{align}
The next simplification consists in approximating the quadratic form $\sfrac{\inprod{x_i, \Ia\hat{G}_i(z)x_i}}{n}$ by $\mathcal{V}(z)$ up to corrections with vanishing moments.

\paragraph{Approximation of $\sfrac{\inprod{x_i, \Ia\hat{G}_i(z)x_i}}{n}$ ---}  This step proceeds through successive approximations by $\sfrac{\inprod{x_i, \Ia G_\smi(z)x_i}}{n}$, $\Ia\sfrac{1}{n}\tr{G_\smi(z)}$ and $\Ia \mathfrak{S}(z)$. First, note that
\begin{align}
    \sup_{i\in \enm{n}}\module{\frac{\inprod{x_i, \Ia(\hat{G}_i(z)-G_\smi (z))x_i}}{n}}&\le \frac{4}{\lambda^2}(2+\sqrt{\sfrac{1}{\alpha}})^4\norm{\partial^3\ell}_\infty \sup_{i\in \enm{n}}\sup_{j\ne i }\module{r_{j,\smi}-r_j}\\
    &=O_{L_k}\left(\frac{\polylog(n)}{\sqrt{n}}\right),
\end{align}
using Proposition \ref{prop: LOO_approxs} in the second equality. We now approximate the quadratic form by a tracial quantity. Decomposing the resolvent $G_\smi(z)$ into real and imaginary parts yields
\begin{align}
   & \sup_{i\in\enm{n}}\module{\frac{\inprod{x_i, G_\smi(z)x_i}}{n}-\frac{1}{n}\tr[G_\smi(z)]}\\
   &\le   \sup_{i\in\enm{n}}\module{\Im{z}}\module{\frac{\inprod{x_i, G_\smi(z)G_\smi(z)^\dagger x_i}}{n}-\frac{1}{n}\tr[G_\smi(z)G_\smi(z)^\dagger]}\\
    &+ \sup_{i\in\enm{n}}\module{\Re{z}}\module{\frac{\inprod{x_i, G_\smi(z)G_\smi(z)^\dagger x_i}}{n}-\frac{1}{n}\tr[G_\smi(z)G_\smi(z)^\dagger]}\\
    &+\sup_{i\in\enm{n}}\module{\frac{\inprod{x_i, G_\smi(z)G_\smi(z)^\dagger H_\smi x_i}}{n}-\frac{1}{n}\tr[G_\smi(z)G_\smi(z)^\dagger H_\smi]}.
\end{align}
Since $G_\smi(z)G_\smi(z)^\dagger H_\smi\succeq 0$ and $G_\smi(z)G_\smi(z)^\dagger \succeq 0$, and 
\begin{align}
\norm{G_\smi(z)G_\smi(z)^\dagger}\le \frac{1}{\Im{z}^2}, && \norm{G_\smi(z)G_\smi(z)^\dagger H_\smi} \le     \frac{1}{\Im{z}^2} \left(\lambda +\norm{\partial^2\ell}_\infty\frac{\norm{X}^2}{n} \right),
\end{align}
one can apply Lemma \ref{lemma:quad}, yielding the concentration
\begin{align}
    \sup_{i\in\enm{n}}\module{\Ia \frac{\inprod{x_i, G_\smi(z)x_i}}{n}-\frac{1}{n}\tr[\Ia G_\smi(z)]} =O_{L_{k}}\left(
   \frac{\polylog(n)}{\sqrt{n}}
    \right).
\end{align}
Finally,
\begin{align}
    \sup_{i\in \enm{n}}\module{\Ia \frac{1}{n}\tr[G_\smi(z)-G(z)]}=O_{L_k}\left(\frac{\polylog(n)}{\sqrt{n}}\right),
\end{align}
again from Proposition \ref{prop: LOO_approxs}. Finally, remember from Lemma \ref{lem:IaSz_concentrates} the concentration 
\begin{align}
\module{\Ia \mathfrak{S}(z)-\mathcal{V}(z)} =O_{L_k}\left(\frac{\polylog(n)}{\sqrt{n}}\right) . 
\end{align}
Chaining all these successive approximations together, we have hence established the pointwise concentration
\begin{align}
    \sup_{i\in\enm{n}}\module{\frac{\Ia\inprod{x_i, \hat{G}_i(z)x_i}}{n}-\mathcal{V}(z)}=O_{L_k}\left(\frac{\polylog(n)}{\sqrt{n}}\right).  \label{eq:xhatGx}
\end{align}
One would like to transfer this control to the summand in \eqref{eq:after_Wood}. Expounding the corresponding correction term, 
\begin{align}
    &\sup_{i\in\enm{n}}\module{\frac{\partial^2\ell_i(\tilde{r}_{i,i})\frac{ \inprod{x_i, \Ia\hat{G}_i(z)x_i}}{n}}{1+\partial^2\ell_i(\tilde{r}_{i,i})\frac{\inprod{x_i, \Ia  \hat{G}_i(z)x_i}}{n}}-\frac{\partial^2\ell_i(\tilde{r}_{i,i})\mathcal{V}(z)}{1+\partial^2\ell_i(\tilde{r}_{i,i})\mathcal{V}(z)}}\\
    &\le \sup_{i\in\enm{n}} \frac{1}{\module{1+\partial^2\ell_i(\tilde{r}_{i,i})\frac{\inprod{x_i, \Ia  \hat{G}_i(z)x_i}}{n}}}\sup_{i\in\enm{n}} \frac{\partial^2\ell_i(\tilde{r}_{i,i})}{\module{1+\partial^2\ell_i(\tilde{r}_{i,i})\mathcal{V}(z)}}\sup_{i\in\enm{n}}\module{\frac{\Ia\inprod{x_i, \hat{G}_i(z)x_i}}{n}-\mathcal{V}(z)}. \label{eq:bounding_denom}
\end{align}
Therefore, it remains to lower-bound the two terms in the denominator. The first term was bounded in \eqref{eq:bound_denom_quad}.

\paragraph{Upper bound on $\frac{\partial^2\ell_i(\tilde{r}_{i,i})}{\module{1+\partial^2\ell_i(\tilde{r}_{i,i})\mathcal{V}(z)}}$---} We now turn to the second term in \eqref{eq:bounding_denom}. On the one hand, the modulus of $1+\partial^2\ell_i(\tilde{r}_{i,i})\mathcal{V}(z)$ is lower-bounded by the absolute value of its imaginary part, namely
\begin{align}
    \module{1+\partial^2\ell_i(\tilde{r}_{i,i})\mathcal{V}(z)}&\ge \partial^2\ell_i(\tilde{r}_{i,i})\module{\Im{z}}\Ea{\Ia \sfrac{1}{n}\tr[G(z)G(z)^\dagger]}\\
    &\ge \frac{1}{\alpha}\P{\mathcal{A}}\partial^2\ell_i(\tilde{r}_{i,i})\module{\Im{z}}  \frac{1}{\left(\module{z}+(\lambda+\norm{\partial^2\ell}_\infty (2+\sqrt{\sfrac{1}{\alpha}})^2)\right)^2}\\
    &\ge \frac{1}{2\alpha} \frac{\module{\Im{z}}}{\left(\module{z}+(\lambda+\norm{\partial^2\ell}_\infty (2+\sqrt{\sfrac{1}{\alpha}})^2)\right)^2} \partial^2\ell_i(\tilde{r}_{i,i})
\end{align}
Therefore, if 
\begin{align}
    \partial^2\ell_i(\tilde{r}_{i,i})\ge \frac{\lambda^2 \alpha}{8}\frac{1}{\module{z}+\lambda+\norm{\partial^2\ell}_\infty (2+\sqrt{\sfrac{1}{\alpha}})^2}
\end{align}
then
\begin{align}
    \frac{1}{\module{1+\partial^2\ell_i(\tilde{r}_{i,i})\mathcal{V}(z)}}\le \frac{\sfrac{16}{\lambda^2}\left(\module{z}+\lambda+\norm{\partial^2\ell}_\infty (2+\sqrt{\sfrac{1}{\alpha}})^2\right)^3}{\module{\Im{z}}}.
\end{align}
On the other hand, if 
\begin{align}
    \partial^2\ell_i(\tilde{r}_{i,i})< \frac{\lambda^2\alpha}{8}\frac{1}{\module{z}+\lambda+\norm{\partial^2\ell}_\infty (2+\sqrt{\sfrac{1}{\alpha}})^2},
\end{align}
we bound  the modulus of $1+\partial^2\ell_i(\tilde{r}_{i,i})\mathcal{V}(z)$ by the absolute value of its real part:
\begin{align}
    \module{1+\partial^2\ell_i(\tilde{r}_{i,i})\mathcal{V}(z)}\ge \module{1+\partial^2\ell_i(\tilde{r}_{i,i})\Ea{\Ia \sfrac{1}{n}\tr[G(z)G(z)^\dagger(H-\Re{z}I_d)]}}.
\end{align}
Since
\begin{align}
    \module{\Ea{\Ia \sfrac{1}{n}\tr[G(z)G(z)^\dagger(H-\Re{z}I_d)]}}\le \frac{4}{\lambda^2 \alpha}\left[\lambda+\norm{\partial^2\ell}_\infty (2+\sqrt{\sfrac{1}{\alpha}})^2+\module{z}\right],
\end{align}
one has in this case
\begin{align}
    \module{1+\partial^2\ell_i(\tilde{r}_{i,i})\mathcal{V}(z)}\ge \frac{1}{2}.
\end{align}
Thus, irrespective of the value of $\partial^2\ell(\tilde{r}_{i,i})$, one has the bound
\begin{align}
    \frac{1}{\module{1+\partial^2\ell_i(\tilde{r}_{i,i})\mathcal{V}(z)}}\le  \frac{\sfrac{16}{\lambda^2}\left(\module{z}+\lambda+\norm{\partial^2\ell}_\infty (2+\sqrt{\sfrac{1}{\alpha}})^2\right)^3}{\module{\Im{z}}}\vee 2.\label{eq:bound_deno_2}
\end{align}

\paragraph{End of the proof of Lemma \ref{lem:characVz} ---} Returning to \eqref{eq:bounding_denom}, one thus has pointwise in $z$ that
\begin{align}
    \sup_{i\in\enm{n}}\module{\frac{\partial^2\ell_i(\tilde{r}_{i,i})\frac{ \inprod{x_i, \Ia\hat{G}_i(z)x_i}}{n}}{1+\partial^2\ell_i(\tilde{r}_{i,i})\frac{\inprod{x_i, \Ia  \hat{G}_i(z)x_i}}{n}}-\frac{\partial^2\ell_i(\tilde{r}_{i,i})\mathcal{V}(z)}{1+\partial^2\ell_i(\tilde{r}_{i,i})\mathcal{V}(z)}}=O_{L_k}\left(\frac{\polylog(n)}{\sqrt{n}}\right).
\end{align}
Returning to the identity \eqref{eq:after_Wood} and taking expectation,
\begin{align}
    \frac{1}{\alpha}-(\lambda-z)\mathcal{V}(z)=&\frac{1}{n}\sum\limits_{i\in\enm{n}}\Ea{\frac{\partial^2\ell_i(\tilde{r}_{i,i})\mathcal{V}(z)}{1+\partial^2\ell_i(\tilde{r}_{i,i})\mathcal{V}(z)}}+O\left(\frac{\polylog(n)}{\sqrt{n}}\right)\\
    =&\Ea{\frac{\partial^2\ell_1(\tilde{r}_{1,1})\mathcal{V}(z)}{1+\partial^2\ell_1(\tilde{r}_{1,1})\mathcal{V}(z)}}+O\left(\frac{\polylog(n)}{\sqrt{n}}\right),
\end{align}
where in going to the second line we leveraged the observation that by permutation symmetry of the problem, all expectations are equal.
The expectation over the surrogate variable $\tilde{r}_{1,1}$ can further be converted into a Gaussian integral by using the distributional approximation of Lemma D.3 of \cite{cui2026asymptotic}. More precisely, equipped with the shorthand $r=\prox_{V^{(1)}\ell(\cdot, g_2)}\left(
  g_1
  \right)$,
\begin{align}
    \Ea{\frac{\partial^2\ell_1(\tilde{r}_{1,1})\mathcal{V}(z)}{1+\partial^2\ell_1(\tilde{r}_{1,1})\mathcal{V}(z)}}
    &=\Ea{\frac{\partial^2\ell\left(r,g_2\right)\mathcal{V}(z)}{1+\partial^2\ell\left(r,g_2\right)\mathcal{V}(z)}}+O\left(\frac{\polylog(n)}{n^{\frac{1}{4}}}\right),
\end{align}
where $(g_1, g_2)\sim\mathcal{N}(0_2, Q)$, and the covariance $Q$ satisfies \eqref{eq:charac_Q} \citep{cui2026asymptotic}.
Hence, finally 
\begin{align}
    \frac{1}{\alpha}-(\lambda-z)\mathcal{V}(z)=\Ea{\frac{\partial^2\ell\left(r,g_2\right)\mathcal{V}(z)}{1+\partial^2\ell\left(r,g_2\right)\mathcal{V}(z)}}+O\left(\frac{\polylog(n)}{n^{\frac{1}{4}}}\right),
\end{align}
completing the proof.
\end{proof}

\subsection{Convergence of the distribution of DFBETAs}
\label{app:influence}
 Appendices \ref{app:concentration} and \ref{app:deterministic} provide a characterization in the high-dimensional limit $n\asymp d$ of a set of summary statistics $ V^{(k)}$, \eqref{eq:Vk}, which subsume key statistical features of the Hessian $H$. Equipped with those descriptors, the present Appendix builds upon these characterizations to deduce the asymptotic distribution of sample influences \eqref{eq:def_DFBETA}. 

\subsubsection{Comments on the risk normalization}

Before proceeding, it is crucial to make the observation that in the definition of $\hat{w}_\smi$ \eqref{eq:wsmi}, leveraged repeatedly throughout the leave-one-out analysis presented in Appendices \ref{app:concentration} and \ref{app:deterministic}, the sum of losses retains a $\sfrac{1}{n}$ normalization, as opposed to the more natural $\sfrac{1}{n-1}$ scaling in the definition of $\hat{w}_\soi$ \eqref{eq:soi_ERM}. We thus first establish approximation results related to this difference in rescaling. Importantly, observe that the change in rescaling is equivalent to perturbing the ridge regularization $\lambda$ by a $O(\sfrac{1}{n})$ term, as we formalize below. Let us recall the definitions
\begin{align}
    &\hat{w}=\argmin_{w\in\R^d} \frac{1}{n}\sum\limits_{j\in\enm{n}} \ell(\inprod{x_j, w}, y_j) +\frac{\lambda}{2}\norm{w}^2,\\
    &\hat{w}_\smi=\argmin_{w\in\R^d} \frac{1}{n}\sum\limits_{j\in\enm{n}\smi} \ell(\inprod{x_j, w}, y_j) +\frac{\lambda}{2}\norm{w}^2,\\
    &\hat{w}_\soi=\argmin_{w\in\R^d} \frac{1}{n-1}\sum\limits_{j\in\enm{n}\smi} \ell(\inprod{x_j, w}, y_j) +\frac{\lambda}{2}\norm{w}^2.
\end{align}
We also introduce
\begin{align}
&\hat{w}_\so=\argmin_{w\in\R^d} \frac{1}{n-1}\sum\limits_{j\in\enm{n}} \ell(\inprod{x_j, w}, y_j) +\frac{\lambda}{2}\norm{w}^2.
\end{align}
Retracing all steps in the proofs of Appendix \ref{app:concentration} and \ref{app:deterministic} reveals that the change of normalization is inconsequential, and all the results above, in particular Proposition \ref{prop: LOO_approxs} and Lemma \ref{lem:characVz}, apply equally to the pairs $(\hat{w}, \hat{w}_\smi)$ and $(\hat{w}_\so, \hat{w}_\soi)$ respectively. However, because the influence metric $\DFBETA_i$ \eqref{eq:def_DFBETA} compares estimators with \textit{different} normalizations, a finer control of the discrepancies \textit{between} these pairs of estimators is called for. This is the object of the following Lemma.

\begin{lemma}[Approximation after rescaling] \label{lemma:approx_norm}
One has
\begin{align}
    \norm{\hat{w}-\hat{w}_\so}=O_{L_k}\left(\frac{1}{n}\right).
\end{align}
As a consequence, denoting $r_{j,\so}=\inprod{x_j, \hat{w}_\so}$,
\begin{align}
    \sup_{j\in\enm{n}}\module{r_{j,\so}-r_j}=O_{L_k}\left(\frac{\polylog(n)}{\sqrt{n}}\right).
\end{align}
\end{lemma}
\begin{proof}
    Observe that one can rewrite the empirical risk minimized by $\hat{w}_\so$ as
    \begin{align}
        \hat{w}_\so=\argmin_{w\in\R^d} \frac{1}{n}\sum\limits_{j\in\enm{n}} \ell(\inprod{x_j, w}, y_j) +\frac{\lambda}{2}\left(1-\frac{1}{n}\right)\norm{w}^2,
    \end{align}
absorbing the difference in normalization into the ridge regularization. From the optimality conditions, it follows that
\begin{align}
    \lambda (\hat{w}-\hat{w}_\so)+\frac{\lambda}{n}\hat{w}_\so +\frac{1}{n}\sum\limits_{j\in\enm{n}} \partial^2\ell_j(\check{r}_j)x_jx_j^\top (\hat{w}-\hat{w}_\so)=0,
\end{align}
for some $\check{r}_j\in(r_{j,\so},r_j)$. Thus
\begin{align}
    \hat{w}-\hat{w}_\so=-\frac{\lambda }{n}\left[\frac{1}{n}\sum\limits_{j\in\enm{n}} \partial^2\ell_j(\check{r}_j)x_jx_j^\top+\lambda I_d\right]^{-1}\hat{w}_\so. \label{eq:diff_renorm}
\end{align}
A straightforward adaptation of Lemma \ref{lemma:normw} shows that
\begin{align}
    \norm{\hat{w}_\so}\le 2M,
\end{align}
from which it follows that
\begin{align}
    \norm{\hat{w}-\hat{w}_\so}\le \frac{2M}{n}.
\end{align}
\end{proof}

\begin{remark}
    An identical result naturally follows for all $i\in\enm{n}$ at the level of the leave-$i-$out estimators, namely 
    \begin{align}
        \norm{\hat{w}_\soi-\hat{w}_\smi}=O\left(\frac{1}{n}\right).
    \end{align}
\end{remark}

\subsubsection{Proof of Proposition \ref{prop:DFBETA}}

A stronger convergence result can be established for the DFBETA metric \eqref{eq:def_DFBETA}, which measures the change in estimator, in terms of Euclidean norm, when a sample is removed from the dataset. More precisely, for any index $i\in\enm{n}$, we remind the definition
\begin{align}
    \DFBETA_i=n \norm{\hat{w}-\hat{w}_\soi}^2.
\end{align}

The following Proposition establishes that the \textit{empirical distribution }over the training set concentrates to a limiting distribution, shaped by the summary statistics $V^{(2)}, Q^{(0)}$ characterized in Appendix \ref{app:deterministic}.

\begin{proposition}
    For any bounded and twice-differentiable test function $f$ with bounded derivatives, the empirical average of DFBETAs processed by $f$ over the train set concentrates as
\begin{align}
    \frac{1}{n}\sum\limits_{i\in\enm{n}}f\left(\DFBETA_i\right)= \Ea{f\left(\left(\frac{\prox_{V^{(1)}\ell(\cdot, g_2)}(g_1)-g_1}{V^{(1)}}\right)^2 V^{(2)}\right)}+O_{L_2}\left(\frac{\polylog(n)}{n^{\frac{1}{4}}}\right).
\end{align}
\end{proposition}

\begin{proof}
    The proof proceeds in two steps. First observe from Lemma \ref{lemma:approx_norm} that
    \begin{align}
 \DFBETA_i=n \norm{\hat{w}_\so-\hat{w}_\soi}^2  +O_{L_k}\left(\frac{1}{\sqrt{n}}\right) .
    \end{align}
    We then  establish that
    \begin{align}
        \sup_{i\in\enm{n}}\module{n \norm{\hat{w}_\so-\hat{w}_\soi}^2- \partial\ell_i(r_{i, \so})^2V^{(2)}}=O_{L_2}\left(\frac{\polylog(n)}{\sqrt{n}}\right).
    \end{align}
This follows from Lemma \ref{lemma:quad}, and the concentration result \ref{lem:Vz_concentrates}. Thus, 
\begin{align}
    \frac{1}{n}\sum\limits_{i\in\enm{n}}f\left(\DFBETA_{i}\right)&=\frac{1}{n}\sum\limits_{i\in\enm{n}}f\left(\partial\ell_i(r_{i, \so})^2V^{(2)}\right)+ O_{L_2}\left(\frac{\polylog(n)}{\sqrt{n}}\right)\\
    &=\Ea{f\left(\partial\ell_1(r_{1, \so})^2V^{(2)}\right)}+ O_{L_2}\left(\frac{\polylog(n)}{\sqrt{n}}\right)\\
    &=\Ea{f\left(\partial\ell_1(\Tilde{r}_{1,\so})^2V^{(2)}\right)}+ O_{L_2}\left(\frac{\polylog(n)}{\sqrt{n}}\right)
\end{align}
using the concentration Lemma F.1 in \cite{cui2026asymptotic}. The expectation can finally be computed using Lemma D.3 of \cite{cui2026asymptotic}, proving the claim.
\end{proof}

\subsection{Auxiliary lemmas}
\label{app:auxiliary}
For completeness and self-containedness, we reproduce below classical auxiliary lemmas that are repeatedly leveraged throughout the main body of the proof. A proof of these standard results can be collected from e.g. \cite{el2018impact, karoui2013asymptotic, cui2026asymptotic, vershynin2018high}.

\begin{lemma}[Norm of the estimator] \label{lemma:normw}
The norm of the estimator $\hat{w}$ is bounded almost surely as
\begin{align}
    &\norm{\hat{w}}\le M
\end{align}
\end{lemma}

\begin{proof}
    From the definition of $\hat{w}$,
\begin{align}
    \frac{\lambda}{2}\norm{\hat{w}}^2+\frac{1}{n}\sum\limits_{j\in \enm{n}}\ell_{j}(r_j)\le \frac{1}{n}\sum\limits_{j\in \enm{n}}\ell_{j}(0),
\end{align}
since the right-hand side corresponds to the empirical risk evaluated at $0_d$.
Thus
\begin{align}
    \norm{\hat{w}}^2 &\le \frac{2}{\lambda}\norm{\tilde{\ell}}_\infty\le M^2.
\end{align}
\end{proof}

\begin{lemma}[Operator norm of the empirical covariance]\label{lem:cov}
Let $\hat{\Sigma}=\frac{X^\top X}{n}$ be the empirical covariance matrix. Then
\begin{align}
    \norm{\hat{\Sigma}}=  O_{L_k} \left(\polylog(n)\right).
\end{align}
and 
\begin{align}
    \frac{\norm{X}}{\sqrt{n}}=O_{L_k} \left(\polylog(n)\right).
\end{align}
The same bounds hold for the leave-one-row-out matrix $X_\smi$.
\end{lemma}

\begin{lemma}[Norm of the samples]\label{lemma:normx}
One can control the sample norms as
\begin{align}
    \sup_{i\in\enm{n}}\frac{\norm{x_i}}{\sqrt{n}}=O_{L_k}(\polylog(n)).
\end{align}
Furthermore,
\begin{align}
    \sup_{i\in \enm{n}}\frac{\norm{X_\smi}}{\sqrt{n}}=O_{L_k}(\polylog(n))
\end{align}
\end{lemma}

\begin{lemma}[Quadratic forms] \label{lemma:quad}
Let  $M_\smi \in \R^{d\times d}$ be a matrix which does not depend on $x_i$, but generically depends on all other samples.  Then for any $k\in [\lceil 4 \log(n) \rceil]$,
\begin{align}
       \sup_{i\in \enm{n}}\sup_{j\ne i}\left|\frac{1}{n}\inprod{x_j,M_\smi x_i}\right| =O_{L_k}\left(\frac{\sup_{i\in\enm{n}}\norm{M_\smi}\polylog(n)}{\sqrt{n}}\right).
\end{align}    
Furthermore \citep{el2018impact}, 
\begin{align}
    \sup_{i\in \enm{n}}\module{\frac{1}{n}\inprod{x_i,M_\smi x_i} - \frac{1}{n}\tr[M_\smi]}=O_{L_k}\left(\sup_{i\in\enm{n}}\norm{M_\smi}\frac{\polylog(n)}{\sqrt{n}}\right)
\end{align}
\end{lemma}

\begin{lemma}[Sup of one-dimensional projections]\label{lem:sup_s}
Let $v_\smi$ be vectors that do not depend on $x_i$, but can generically depend on all other samples.  Then for any $k\in [\lceil2\log(n)\rceil]$,
\begin{align}
    \sup_{i\in\enm{n}}\left|\inprod{v_\smi,x_i}\right|= O_{L_k}\left(\sup_{i\in\enm{n}}\norm{v_\smi} \sqrt{\polylog(n)}\right).
\end{align}
Let $\{\{u_{j,\smi}\}_{j\ne i}\}_{i\in\enm{n}}$ be a family of $n(n-1)$ vectors such that $u_{j,\smi}$ does not depend on $x_i$, but can generically depend on the other samples. Then for any $k\in [\lceil 4 \log(n) \rceil]$ 
\begin{align}
    \sup_{i\in\enm{n}}\sup_{j\ne i}\left|\inprod{u_{j,\smi},x_i}\right|= O_{L_k}\left(\sup_{i\in\enm{n}}\sup_{j\ne i} \norm{u_{j,\smi}}\sqrt{\polylog(n)}\right).
\end{align}
\end{lemma}

\begin{lemma}[$X\to \hat{w}(X)$ is differentiable]\label{lem:differentiability} The function $f:\R^{n\times d}\to \R^d$ with $f(X)=\hat{w}$ is differentiable.    
\end{lemma}
\begin{proof}
    Let us introduce the function $F:\R^{n\times d}\times \R^d\to \R^d$ defined as 
    \begin{align}
        F(X,w)=\lambda w+\frac{1}{n}\sum\limits_{i\in \enm{n}}\partial\ell(\inprod{w,x_i}, \inprod{\beta, x_i})x_i.
    \end{align}
For any $X$, $F(X, \hat{w}(X))=0$ by definition of the ERM minimizer. Moreover, the Jacobian is 
\begin{align}
   \nabla_wF|_{X, \hat{w}(X)} =H \succ 0.
\end{align}
Thus, by the implicit function theorem, the function $f:Z\to \hat{w}(Z)$ is differentiable.

\end{proof}

\subsection{Extensions}
\label{app:extension}
Proposition \ref{prop:DFBETA} was reported in the main text in the simple setting of a noiseless label generation process \eqref{eq:label_gen} and Gaussian data. We discuss in this Appendix how these results may be extended to a larger class of mixture distributions, and stochastic settings. While we anticipate that the proof should proceed in very close alignment with the one presented in Appendices \ref{app:concentration} and \ref{app:deterministic}, we leave a rigorous proof of this extension to future work, and simply conjecture here the end results.

\paragraph{Superstatistical design --- } We now consider a larger class of data distributions,  that can be defined in superstatistical fashion \citep{el2018impact, adomaityte2024high} as an infinite Gaussian mixture of centered clusters of random variance:
\begin{align}
    x_i \mid \sigma_i \sim \mathcal{N}(0, \sigma_i^2 I_d), \label{eq: Gamma}
\end{align}
with the scale factors $\sigma_i$ being sampled i.i.d from a distribution $\rho$. For simplicity, we assume in the present manuscript that the latter is supported on $[0,1]$. We note that the Gaussian case considered in the main text may be retrieved in the special case $\rho(\cdot)=\delta(\cdot -1)$. 

\paragraph{Stochasticity--- } We also assume the presence of a source of stochasticity in the problem, with each sample $(x_i, y_i)$ being corrupted by a scalar noise $\epsilon_i$, sampled i.i.d. from a density $\mu$ with all moments being bounded. This stochasticity can intervene for instance as label noise, in which case  \eqref{eq:label_gen} becomes
\begin{align}
    y_i=\phi(\inprod{x_i, \beta},\epsilon_i).
\end{align}
To keep the problem as general as possible, we simply consider the ERM
\begin{align}
    \hat{w}=\argmin_{w\in\R^d} \frac{1}{n}\sum\limits_{i\in\enm{n}} \ell(\inprod{x_i, w}, \inprod{x_i, \beta}, \epsilon_i) +\frac{\lambda}{2}\norm{w}^2,\label{eq:extension}
\end{align}
with an extended loss function $\ell:\R^3 \to \R_+$ taking the noise $\epsilon_i$ as a last argument. We extend in the same fashion the leave-$i-$out ERM \eqref{eq:soi_ERM}. We assume Assumption \ref{ass:loss} holds uniformly over the last argument of $\ell$.\\

One is now in a position to revisit Proposition \ref{prop:DFBETA}. We form the following conjecture:

\begin{conjecture}[Noisy case]\label{Conj:extension}
     Let $\mathcal{D}=\{x_i, y_i\}_{i\in\enm{n}}$ be a randomly sampled dataset, and $\{\DFBETA_i\}_{i\in\enm{n}}$ be the $\DFBETA$ influences \eqref{eq:def_DFBETA} associated with the ERM \eqref{eq:full_ERM}. In the asymptotic limit $n,d\to \infty$ with fixed ratio $\alpha:=\sfrac{n}{d}$, the empirical distribution
    \begin{align}
        \hat{\nu}_D=\frac{1}{n}\sum\limits_{i\in\enm{n}}\delta_{\DFBETA_i}
    \end{align}
concentrates around the deterministic measure
\begin{align}
    \nu_D= \varphi_D~\sharp~\left[\mathcal{N}(0_2, Q)\times \rho\times \mu\right],
\end{align}
where the map $\varphi_D:\R^4\to\R_+$ is defined by
\begin{align}
    \varphi_D(g, \sigma, \epsilon)=\left(\frac{\prox_{\sigma^2 V^{(1)}\ell(\cdot, \sigma g_2,\epsilon)}(\sigma g_1)-\sigma g_1}{\sigma^2 V^{(1)}}\right)^2 \sigma^2 V^{(2)}.
\end{align}
The Stieltjes transform $\mathcal{V}(\cdot)$ now satisfies for all $z\in \Gamma, \Im{z}\ne 0$ the self-consistent equation 
 \begin{align}
   \frac{1}{\alpha}-(\lambda-z)\mathcal{V}(z)-\Ea{\frac{\sigma^2\partial^2\ell(r,\sigma g_2,\epsilon)\mathcal{V}(z)}{1+\sigma^2\mathcal{V}(z)\partial^2\ell (r,\sigma g_2,\epsilon)}
}=O\left(\frac{\polylog(n)}{n^{\frac{1}{4}}}\right)
\end{align}
The components of the covariance $Q$ on the other hand now satisfy
\begin{align}
        &  
        q=-\frac{1}{\lambda}\Ea{\partial\ell(r,\sigma g_2, \epsilon )r}
        +O\left(\frac{\polylog(n)}{n^{\frac{1}{4}}}\right),\\
    &      
       m =-\frac{1}{\lambda}\Ea{\partial\ell(r,\sigma g_2,\epsilon)\sigma g_2}
      + O\left(\frac{\polylog(n)}{n^{\frac{1}{4}}}\right).
\end{align}
Above, we used the shorthand $r=\prox_{\sigma^2 V^{(1)}\ell(\cdot, \sigma g_2, \epsilon)}(\sigma g_1)$. Across the above displays, expectations bear over $g\sim\mathcal{N}(0_2, Q), \epsilon\sim\mu, \sigma\sim\rho$.
\end{conjecture}

\newpage

\section{Details on numerical experiments}
\label{app:numerics}
In this Appendix, we provide further details on the real data experiments reported in Fig.\,\ref{fig:Real_data}. The code is provided on this online \href{https://github.com/HugoCui/HD_influences/tree/main}{repository}.

\begin{table}[h]
\centering
\begin{tabular}{|l|c|c|}
\hline
 & MNIST \citep{lecun1998gradient} & Chest X-rays \citep{kermany2018identifying} \\ \hline
$\module{\mathcal{D}_{\rm NN}}$ & $3000$ & $1000$ \\ \hline
$\module{\mathcal{D}_{\rm \beta}}$ & $50000$ & $3000$ \\ \hline
$\module{\mathcal{D}_{\rm train}}$ & $500$ & $500$ \\ \hline
$\module{\mathcal{D}_{\rm test}}$ & $9500$ &$1340$ \\ \hline
$\sharp $ epochs & $10$ & $50$ \\ \hline
width & $1000$ &$2000$ \\ \hline
\end{tabular}
\caption{Dataset sizes and training hyperparameters used in the experiments reported in Fig.\,\ref{fig:Real_data}.}
\label{tab:exp}
\end{table}

We consider two learning tasks on real datasets:
\begin{itemize}
    \item Even/odd classification of MNIST \citep{lecun1998gradient} digits. Images in dimension $d=784$ representing handwritten even (resp. odd) digits were labeled $+1$ (resp $-1$).
    \item Pneumonia diagnosis from chest X-rays \citep{kermany2018identifying}. X-ray scan images (cropped to dimension $d=1600$) corresponding to positive (resp. negative) patients receive label $+1$ (resp. $-1$).
\end{itemize}

Each dataset was partitioned into four subsets $\mathcal{D}_{\mathrm{NN}}, \mathcal{D}_\beta, \mathcal{D}_{\rm train}, \mathcal{D}_{\rm test}$. Details on their respective cardinals are summarized in Table \ref{tab:exp}. The following procedure was then carried out:
\begin{enumerate}[noitemsep]
    \item A 3-layer neural network with ReLU activation and batch  normalization is trained using the Adam optimizer \citep{kingma2014adam} with default \texttt{Pytorch} \citep{paszke2017automatic} initialization, for a number of epochs specified in Table \ref{tab:exp}, on the dataset $\mathcal{D}_{\mathrm{NN}}$. The binary cross-entropy loss is employed. The width $p$ of the network was chosen for all hidden layer as $p=1000$ for the MNIST dataset and $p=2000$ for the chest X-ray dataset. The resulting neural network feature map $\mathrm{NN} :\R^d\to \R^p$ (with the readout removed) will be used in subsequent steps as a feature extractor to process the datapoints.
    \item A logistic regression with regularization $\lambda=0.01$ is then used to re-train the readout only, on the processed dataset $\{(\mathrm{NN}(x), y)\}_{(x,y)\in \mathcal{D}_\beta}$. The cardinal of $\mathcal{D}_\beta$ is chosen large on purpose, so that the resulting regression weights $\beta$ are close to the best separator of the data in the feature space induced by $\mathrm{NN}(\cdot)$. $\beta$ will henceforth be thought of as the "ground truth" vector, in analogy to \eqref{eq:label_gen}.
    \item Similarly, the estimator $\hat{w}$ learned from a smaller dataset $\mathcal{D}_{\rm train}$ is obtained using logistic regression with regularization $\lambda=0.1$. 
    \item Given the running estimator $\hat{w}$ and (an approximation of) the ground-truth $\beta$, we now ask the question : \textit{which sample has the largest influence on the test accuracy (as measured using the test set $\mathcal{D}_{\rm test}$) when added to $\mathcal{D}_{\rm train}$}? We allow ourselves to synthesize (query) any sample $z\in\mathrm{span}(\hat{w},\beta)$, and assume oracle access to its label, taken to be $y=\mathrm{sign}(\inprod{\beta, z})$. The influence $\IF(z)$ is computed as $|\mathcal{D}_{\mathrm{train}}|$ times the difference in the test errors achieved by logistic regression on $\mathcal{D}_{\rm train}\cup \{z,y\}$, and logistic regression trained only on $\mathcal{D}_{\rm train}$.
\end{enumerate}

\end{document}